\documentclass{article} 
\usepackage{iclr2026_conference,times}

\usepackage{amsmath,amsfonts,bm}

\def\eqref#1{equation~\ref{#1}}

\def\1{\bm{1}}

\DeclareMathAlphabet{\mathsfit}{\encodingdefault}{\sfdefault}{m}{sl}
\SetMathAlphabet{\mathsfit}{bold}{\encodingdefault}{\sfdefault}{bx}{n}

\usepackage[utf8]{inputenc}
\usepackage[T1]{fontenc}
\usepackage{hyperref}
\usepackage{url}
\usepackage{booktabs}
\usepackage{amsfonts}
\usepackage{amsmath}
\usepackage{amssymb}
\usepackage{amsthm}
\usepackage{nicefrac}
\usepackage{microtype}
\usepackage{xcolor}
\usepackage{graphicx}
\usepackage{framed}

\newtheorem{definition}{Definition}
\newtheorem{theorem}{Theorem}
\newtheorem{lemma}{Lemma}
\newtheorem{proposition}{Proposition}

\newtheorem{assumption}{Assumption}
\newtheorem{corollary}{Corollary}
\newtheorem{remark}{Remark}

\title{The Alignment Bottleneck}

\author{Wenjun Cao \\
Independent Researcher \\
\texttt{wenjun.cao.research@gmail.com}
}

\iclrfinalcopy

\begin{document}

\maketitle

\begin{abstract}
Large language models improve with scale, yet feedback-based alignment still exhibits systematic deviations from intended behavior. Motivated by bounded rationality in economics and cognitive science, we view judgment as resource-limited and feedback as a constrained channel; on this basis we model the loop as a two-stage cascade $U\!\to\!H\!\to\!Y$ given $S$, with cognitive capacity $C_{\mathrm{cog}\mid S}$ and average total capacity $\bar C_{\mathrm{tot}\mid S}$. Our main result is a capacity coupled Alignment Performance Interval. It pairs a data size independent Fano lower bound proved on a separable codebook mixture with a PAC--Bayes upper bound whose KL term is controlled by the same channel via $m\,\bar C_{\mathrm{tot}\mid S}$. The PAC--Bayes bound becomes an upper bound on the same true risk when the canonical observable loss is used and the dataset is drawn from the same mixture. Under these matched conditions both limits are governed by a single capacity. Consequences include that, with value complexity and capacity fixed, adding labels alone cannot cross the bound; attaining lower risk on more complex targets requires capacity that grows with $\log M$; and once useful signal saturates capacity, further optimization tends to fit channel regularities, consistent with reports of sycophancy and reward hacking. The analysis views alignment as interface engineering: measure and allocate limited capacity, manage task complexity, and decide where information is spent.
\end{abstract}

\section{Introduction}

Scaling laws continue to improve LLM capabilities \citep{kaplan2020scalinglawsneurallanguage, wei2022emergentabilitieslargelanguage, hestness2017deeplearningscalingpredictable}, but feedback-based alignment shows a tension: instruction-following improves on average, while systematic deviations from intended behavior persist. In practice, feedback-based alignment pipelines have substantially improved instruction following \citep{ouyang_training_2022, bai_constitutional_2022, ziegler_fine-tuning_2020, NEURIPS2023_91f18a12, rafailov_direct_2024, lee2024rlaif}. Nevertheless, models continue to exhibit sycophancy, reward hacking, and inverse scaling on truthfulness \citep{perez_discovering_2023, sharma_towards_2024, pan_effects_2022, denison_sycophancy_2024, lin_truthfulqa_2022, amodei2016concreteproblemsaisafety}. A natural question is whether these patterns partly reflect a structural limit of the human–AI feedback loop.

Motivating evidence spans economics and cognitive science: bounded rationality views decisions as resource-limited and often satisficing \citep{Simon1955Behavioral}; computational and information-theoretic models then show that people compress task representations and trade performance for cognitive cost \citep{lewis_computational_2014, ho_people_2022, ho_efficiency_2020, zenon_information-theoretic_2019, zaslavsky-etal-2021-rate}. Rate–distortion and information-bottleneck perspectives connect these constraints to perception, control, and RL \citep{sims_ratedistortion_2016, ortega_information_2011, Lai2021, arumugam_bayesian_2023, arumugam_deciding_2022}. These results motivate treating feedback as information passing through a bounded system rather than as a noiseless oracle.

With this empirical and theoretical background, we model the feedback loop as a two-stage cascade $U\!\to\!H\!\to\!Y$ given context $S$: latent human values $U$ are first compressed into internal judgments $H$, then articulated as observable signals $Y$. We define the total conditional capacity $\bar C_{\mathrm{tot}\mid S}=\mathbb{E}[\min\{C_{\mathrm{cog}\mid S},\,C_{\mathrm{art}\mid S}\}]$ and highlight the cognitive capacity $C_{\mathrm{cog}\mid S}$ as the typical bottleneck through which value information must pass, and we connect rate--distortion and information-bottleneck ideas, focusing on what fidelity is possible under the human--AI channel’s cognitive capacity \citep{tishby_information_2000, tishby_deep_2015, alemi_deep_2017, Kolchinsky_2019, kawaguchi_how_2023, saxe_information_2019, shwartzziv2017openingblackboxdeep, pmlr-v97-goldfeld19a}.

We establish a capacity link: the same capacity that limits value information entering the data also governs the statistical complexity needed for generalization. On the lower-bound side (Sec.~\ref{sec:unified_lower}), using separable codebooks and Fano, we obtain a data-size–independent information lower bound on true risk,
\[
R_{\mathrm{mix}}(\pi)\ \ge\ (\varepsilon+\Delta)\,\Big(1-\tfrac{\bar C_{\mathrm{tot}\mid S}^{\mathrm{mix}}+\log 2}{\log M}\Big)_+.
\]
On the upper-bound side (Sec.~\ref{sec:upper}), via PAC--Bayes for bounded observable losses and the link between KL complexity and dataset–parameter mutual information \citep{xu_information-theoretic_2017, russo_how_2019, rodriguez-galvez_more_2024, lotfi_non-vacuous_2024, dziugaite2017computingnonvacuousgeneralizationbounds}, we show
\[
\mathbb{E}_{\mathcal{D}}\!\big[\mathrm{KL}(P\|Q)\big]\ \le\ m\,\bar C_{\mathrm{tot}\mid S}\ +\ m\,I(U;S)\ +\ \rho\ +\ \mathrm{KL}\!\big(p(\theta)\|Q\big),
\]
which renders the upper bound explicit for the same channel that defines the converse. Taken together, under the canonical observable loss and under the same codebook mixture used in the converse, we obtain a capacity coupled interval:
\[
(\varepsilon+\Delta)\Big(1-\tfrac{\bar C_{\mathrm{tot}\mid S}^{\mathrm{mix}}+\log 2}{\log M}\Big)_+
\ \le\ R_{\mathrm{mix}}(\pi)\ \le\
\mathbb E_{\theta\sim P}\!\big[\widehat R^{\mathrm{obs}}_m(\theta)\big]
+\sqrt{\tfrac{\mathrm{KL}(P\|Q)+\log(1/\delta)}{2m}},
\]
The KL term is further controlled in expectation by the same channel capacity. To our knowledge, prior analyses have not coupled a Fano-type lower bound and a PAC--Bayes upper bound through a single capacity term of the human–AI channel.

These bounds imply that increasing the dataset size $m$ alone cannot overcome the lower bound when separability and capacity are fixed; achieving a target risk requires capacity that scales with value complexity, which constrains pluralistic or multi-objective alignment; and once a useful signal saturates capacity, powerful optimizers can continue to reduce empirical loss by fitting residual channel regularities, consistent with reports of sycophancy and related behaviors \citep{perez_discovering_2023, sharma_towards_2024}.

\section{Related Work}

\textbf{Feedback Alignment and Systematic Deviations}
Contemporary alignment trains policies to preference signals using feedback-driven pipelines that collect preference data and adjust behavior under varied supervision protocols \citep{ouyang_training_2022, christiano_deep_2017, ziegler_fine-tuning_2020, bai_training_2022, NEURIPS2023_91f18a12, rafailov_direct_2024, lin_limited_2024, bai_constitutional_2022, ethayarajh_kto_2024, guo_controllable_2024, lee2024rlaif, NEURIPS2024_c4e380fb}. Despite gains, models display systematic deviations such as sycophancy, reward hacking, and inverse scaling on truthfulness, and raise sequential concerns including user tampering \citep{perez_discovering_2023, sharma_towards_2024, pan_effects_2022, denison_sycophancy_2024, lin_truthfulqa_2022, evans_user_2023}. Formal accounts situate these deviations in system-level incentives and representation/oversight mismatches across aggregation and interaction protocols \citep{ge_axioms_2024, irving_ai_2018, everitt_reward_2021, ngo_alignment_2024, rane_concept_2024}. Related analyses emphasize dynamic optimization effects in which proxy-reward gains can diverge from target behavior under increasing optimization pressure \citep{gaikwad_murphys_2025}, with empirical scaling in model size and KL budgets and $\sqrt{\mathrm{KL}}$-type ceilings \citep{rafailov2024scaling, mroueh2025information}. We instead take a static source--channel view: the same capacity term---typically governed by $C_{\mathrm{cog}\mid S}$---controls the Fano lower bound and the PAC--Bayes complexity, yielding an optimizer-agnostic interval.

\textbf{Bounded Rationality and Cognitive Constraints}
Originating in economics and organizational theory as bounded rationality and satisficing \citep{Simon1955Behavioral}, subsequent work in cognitive science and information-theoretic decision-making models judgment as resource-limited computation with explicit costs for processing and representation \citep{lewis_computational_2014, ho_cognitive_2022, gottwald_bounded_2019, ortega_information_2011, zenon_information-theoretic_2019, ho_efficiency_2020}. Empirically, people construct simplified task representations and plan under constrained internal state \citep{ho_people_2022}. Rate–distortion accounts capture these bottlenecks in perception and communication \citep{sims_ratedistortion_2016, zaslavsky-etal-2021-rate} and connect to RL and Bayesian decision-making to yield capacity-limited agents \citep{arumugam_bayesian_2023, arumugam_deciding_2022, pmlr-v139-arumugam21a, arumugam_value_2021}; policy compression frames action selection as an information bottleneck \citep{Lai2021}. These lines motivate modeling the feedback pipeline as $U\!\to\!H\!\to\!Y$ with a cognitive-capacity term $C_{\mathrm{cog}\mid S}$ that often forms the binding bottleneck and upper-bounds $I(U;Y\mid S)$ in our analysis.

\textbf{Information Theory in Machine Learning}
The Information Bottleneck program studies compression of irrelevant bits while preserving task-relevant information \citep{tishby_information_2000, tishby_deep_2015, alemi_deep_2017, peng_variational_2020, Kolchinsky_2019}, with continued debate about what “compression’’ measures in deterministic networks \citep{saxe_information_2019, pmlr-v97-goldfeld19a, shwartz-ziv_compress_2023, shwartzziv2017openingblackboxdeep}. Information-theoretic bounds relate hierarchical processing to generalization \citep{kawaguchi_how_2023, he_information-theoretic_2025, NIPS2017_b22b257a} and analyze self-supervised objectives \citep{shwartz-ziv_information-theoretic_2024}. Our use of packings and Fano to build a data-size independent wall follows classical converse techniques and rate–distortion thinking \citep{shannon_coding_1959}, and we interpret residual information learned beyond the true value as channel overfitting \citep{ngampruetikorn_information_2022}. The compression view is also relevant because recent work shows that LLMs are strong general-purpose compressors \citep{deletang_language_2024}. These threads link information budgets and performance. Departing from IB’s focus on compressing internal representations, we instead parameterize alignment by the context-conditioned capacity of the human–AI channel ($\bar C_{\mathrm{tot}\mid S}$) and use it to couple a codebook–Fano wall with a PAC–Bayes ceiling.

\textbf{PAC--Bayes and Mutual Information}
PAC--Bayes provides non-asymptotic generalization guarantees that can remain informative at scale \citep{lotfi_non-vacuous_2024, rodriguez-galvez_more_2024, wu_recursive_2025, leblanc_generalization_2025, picard-weibel_how_2025, dziugaite2017computingnonvacuousgeneralizationbounds, neyshabur2018a, NIPS2001_98c72428, NIPS2017_10ce03a1}. A key development ties the KL term to mutual information between data and parameters \citep{xu_information-theoretic_2017, russo_how_2019}, and the PAC--Bayes Information Bottleneck makes this connection algorithmic by directly regularizing $I(\mathcal D;\theta)$ \citep{wang_pac-bayes_2022}. We ground the abstract KL complexity in a physical constraint of the learning environment: the finite human-feedback capacity $\bar C_{\mathrm{tot}\mid S}$, often dominated by $C_{\mathrm{cog}\mid S}$ in our $U\!\to\!H\!\to\!Y$ model. This yields a capacity-coupled interval in which the same capacity term both limits $I(U;Y\mid S)$ in the Fano floor and controls attainable KL complexity in the PAC--Bayes ceiling. Related mitigations, such as information-bottleneck style reward modeling and behavior-supported methods \citep{miao_inform_2024, dai_mitigating_2025}, and upper-bound–style results \citep{mroueh2025information} are compatible with this view by reallocating or constraining where the limited information budget is spent. By externalizing the KL complexity into the environmental budget $m\,\bar C_{\mathrm{tot}\mid S}+m\,I(U;S)+\rho$ induced by the $U\!\to\!H\!\to\!Y$ channel, our bound aligns the PAC–Bayes term with the same capacity that limits $I(U;Y\mid S)$, a linkage not provided by PAC–Bayes-IB.

\section{Problem Setup}
\label{sec:setup}

To rigorously analyze the Alignment Bottleneck, we model alignment as resource-constrained inference and communication. Following bounded and computational rationality in cognitive science \citep{lewis_computational_2014, ortega_information_2011}, we treat the human feedback provider as a two-stage communication channel, and use channel capacity to quantify the bottleneck. This connects our formulation to the Information Bottleneck and rate–distortion viewpoints \citep{tishby_information_2000, shannon_coding_1959} and to cognitive accounts that frame judgment/externalization as utility–information trade-offs \citep{sims_ratedistortion_2016, zaslavsky-etal-2021-rate}. We next formalize the task, the two-stage channel, and the corresponding capacities.

\subsection{Task, Loss, and Feedback Channel}
\label{subsec:task-loss-channel}

\begin{definition}[Task and Observable Loss]
\label{def:loss}
Let $S$ denote publicly observable context, $U$ the latent task target (``what humans truly want''), and $Y$ the human feedback emitted through a finite-capacity channel. 
A learner outputs an action $\hat a=\pi(Y,S)\in\mathcal A$ using a decoder $\pi$. 
The task loss is a \emph{bounded} measurable function $\ell:\mathcal U\times\mathcal A\to[0,1]$ (such as 0--1 loss, pairwise ranking loss mapped to $[0,1]$, or a truncated and normalized MSE; see Appx.~\ref{app:truncation}).
The (population) risk is
\begin{equation}
\label{eq:risk}
R(\pi)\triangleq \mathbb E\big[\ell(U,\pi(Y,S))\big].
\end{equation}
\end{definition}

\begin{definition}[Human Channel Families]
\label{def:families}
We model the human-in-the-loop communication by a cascade $U\!\to\!H\!\to\!Y$ given $S$. 
The \emph{cognitive stage} uses a conditional kernel $p(h\mid u,s)\in \mathcal F_{\mathrm{cog}}$, and the \emph{articulation stage} uses $p(y\mid h,s)\in \mathcal F_{\mathrm{art}}$. 
The learner observes only $(Y,S)$, not $H$.
\end{definition}

This two-stage cascade $U\!\to\!H\!\to\!Y$ treats the human as a finite-capacity communication channel. Evidence from cognitive science shows that human judgment is resource-bounded and therefore compressive rather than perfect retrieval \citep{lewis_computational_2014, zenon_information-theoretic_2019, ortega_information_2011, gottwald_bounded_2019}; people construct task-specific construals that trade representational complexity for utility \citep{ho_people_2022}, which provides a concrete mechanism for the $U\!\to\!H$ bottleneck and aligns with rate–distortion views of perception and communication \citep{sims_ratedistortion_2016, zaslavsky-etal-2021-rate}. Beyond description, bounded-rationality formalisms operationalize these limits: policy selection itself can be cast as an information bottleneck \citep{Lai2021}, and rate–distortion–constrained learning appears in bandits and then full RL, culminating in a common Bayesian/RL view of capacity-limited behavior \citep{pmlr-v139-arumugam21a, arumugam_value_2021, arumugam_deciding_2022, arumugam_bayesian_2023}. We adopt this source–channel lens and treat the finite cognitive capacity $C_{\mathrm{cog}\mid S}$ as an often binding bottleneck through which value information must pass.

\begin{assumption}[Per-stage Feasibility]
\label{assump:per-stage}
All admissible systems considered in this paper satisfy $p(h\mid u,s)\in\mathcal F_{\mathrm{cog}}$ and $p(y\mid h,s)\in\mathcal F_{\mathrm{art}}$ for almost every $(u,s)$, with the two stages independent across i.i.d.\ samples.
\end{assumption}

\subsection{Context-Conditional Capacities}
\label{subsec:capacities}

\begin{definition}[Cognitive Capacity]
\label{def:C_cog}
For each $s$, define
\begin{equation}
\label{eq:def_C_cog}
C_{\mathrm{cog}\mid S}(s)\ \triangleq\ \sup_{p(h\mid u,s)\in\mathcal F_{\mathrm{cog}}}\ I(U;H\mid S=s)\,.
\end{equation}
\end{definition}

\begin{definition}[Articulation Capacity]
\label{def:C_art}
For each $s$, define
\begin{equation}
\label{eq:def_C_art}
C_{\mathrm{art}\mid S}(s)\ \triangleq\ \sup_{p(y\mid h,s)\in\mathcal F_{\mathrm{art}}}\ I(H;Y\mid S=s)\,.
\end{equation}
\end{definition}

\begin{definition}[Total Capacity and Its Average]
\label{def:C_total}
For each $s$, define the per-context total capacity
\begin{equation}
\label{eq:def_C_tot_pointwise}
C_{\mathrm{tot}\mid S}(s)\ \triangleq\ \min\Big\{C_{\mathrm{cog}\mid S}(s),\ C_{\mathrm{art}\mid S}(s)\Big\},
\end{equation}
and its average
\begin{equation}
\label{eq:def_C_tot_avg}
\bar C_{\mathrm{tot}\mid S}\ \triangleq\ \mathbb E_{S}\big[C_{\mathrm{tot}\mid S}(S)\big].
\end{equation}
\end{definition}

\begin{proposition}[Cascade Upper Bound via Data Processing]
\label{prop:cascade}
Under Assumption~\ref{assump:per-stage}, any admissible cascade $U\!\to\!H\!\to\!Y$ forms a Markov chain. By the data processing inequality \citep{shannon_mathematical_1948}, it satisfies for every $s$,
\begin{equation}
\label{eq:cascade_pointwise}
I(U;Y\mid S=s)\ \le\ \min\{I(U;H\mid S=s),\ I(H;Y\mid S=s)\}\ \le\ C_{\mathrm{tot}\mid S}(s),
\end{equation}
and hence, averaging over $S$,
\begin{equation}
\label{eq:cascade_avg}
I(U;Y\mid S)\ \le\ \bar C_{\mathrm{tot}\mid S}.
\end{equation}
\end{proposition}

\begin{remark}[Source-dependent min and achievability]
\label{rem:min_feasible}
Definitions~\ref{def:C_cog}--\ref{def:C_total} impose per-stage feasibility (Assumption~\ref{assump:per-stage}), so \eqref{eq:cascade_pointwise} follows from data processing without additional compatibility assumptions. The conditional quantities $C_{\mathrm{cog}\mid S}(s)$ and $C_{\mathrm{art}\mid S}(s)$ are evaluated under the given source $P(U\mid S=s)$ rather than maximized over inputs; they are therefore analogous to rate–distortion quantities computed for a fixed source \citep{shannon_coding_1959} and depend on $P(U\mid S)$. Finally, $\min\{\sup I(U;H\mid S=s),\,\sup I(H;Y\mid S=s)\}$ is in general a conservative upper bound for the cascade because the per-stage optimizers need not be mutually compatible, so equality (achievability) should not be expected.
\end{remark}

All statements also hold if one replaces the context $S$ by any measurable coarsening $S'=T(S)$: then the interference term $I(U;S)$ is replaced by $I(U;S')\le I(U;S)$ and capacities are recomputed as $\bar C_{\mathrm{tot}\mid S'}$. See Appendix~\ref{app:coarsen}.

\section{Information-Theoretic Lower Bounds}
\label{sec:unified_lower}

Having established our problem model, we derive the first component of the Alignment Performance Interval: an information-theoretic lower bound on the true risk. The bound exposes how task difficulty scales with value complexity and channel capacity. Following the classic minimax methodology from statistical decision theory and information theory \citep{shannon_mathematical_1948}, we construct a family of hard but distinguishable tasks and apply Fano's inequality to show that any algorithm incurs nontrivial error in telling them apart. In our setting, this family is a ``$\Delta$-separable codebook'' of value–action pairs.

\subsection{Separable Codebooks and the Loss--Index Link}
\label{subsec:codebook}

\begin{definition}[$\Delta$-Separable Codebook]
\label{def:separable_codebook}
A collection $\{(u^{(i)},a^{(i)})\}_{i=1}^M\subset\mathcal U\times\mathcal A$ is called a \emph{$\Delta$-separable codebook} for loss $\ell\in[0,1]$ if
\begin{align}
\label{eq:codebook_good}
\ell\big(u^{(i)},a^{(i)}\big) &\le \varepsilon\quad\text{for all }i,
\\
\label{eq:codebook_margin}
\ell\big(u^{(j)},a^{(i)}\big) &\ge \varepsilon+\Delta\quad\text{for all }j\neq i,
\end{align}
for some $\varepsilon\in[0,1-\Delta]$. We write $\mathcal C(M,\Delta,\varepsilon)$ for the set of such codebooks.
\end{definition}

\begin{remark}[How to Build $\mathcal C(M,\Delta,\varepsilon)$ in Practice]
\label{rem:build_codebook}
The construction of such codebooks (packings) is standard for minimax lower bounds in statistical decision theory and information theory \citep{shannon_mathematical_1948}. For 0--1 classification, choose $a^{(i)}$ predicting class $i$, giving $\varepsilon=0$ and $\Delta=1$. For pairwise ranking with 0--1 \emph{pairwise} loss averaged over all $\binom{n}{2}$ pairs, choose $a^{(i)}$ realizing ranking $i$, so $\varepsilon=0$ and any two total orders differ on at least one pair, yielding $\Delta=1/{\binom{n}{2}}$ after normalization to $[0,1]$. For truncated MSE $\ell(u,a)=\min\{\|u-a\|^2/\tau^2,1\}$, take $a^{(i)}=u^{(i)}$ on an $r$-separated packing of $\mathcal U$, then $\varepsilon=0$ and it is consistent with Assump.~\ref{assump:loss_index_link} to use the common margin $\Delta=r^2/(4\tau^2)$ (the prototype cross-loss is $\ge r^2/\tau^2\ge\Delta$, while the Voronoi misclassification loss is $\ge r^2/(4\tau^2)$). See Appx.~\ref{app:packing_constructions}.
\end{remark}

\begin{assumption}[Loss--Index Link via a Measurable Partition]
\label{assump:loss_index_link}
Given a $\Delta$-separable codebook $\{(u^{(i)},a^{(i)})\}_{i=1}^M$, there exists a measurable map $\phi:\mathcal A\times\mathcal S\to [M]$ (an ``index decoder'') such that for all $i$, all $s\in\mathcal S$, and all $a\in\mathcal A$,
\begin{equation}
\label{eq:loss_index_link}
\phi(a,s)\neq i\ \ \Longrightarrow\ \ \ell\big(u^{(i)},a\big)\ \ge\ \varepsilon+\Delta.
\end{equation}
\end{assumption}

\begin{remark}[When Assumption~\ref{assump:loss_index_link} Holds]
\label{rem:assump_link_holds}
For 0--1 classification and pairwise ranking, let $\phi$ return the predicted class/ranking (allowing dependence on $S$ if needed); then \eqref{eq:loss_index_link} holds immediately. For truncated MSE, let $\phi$ be the nearest-prototype Voronoi partition under $\|\cdot\|$ (prototypes may depend on $S$). With an $r$-separated packing, any misclassification implies $\|a-u^{(i)}\|\ge r/2$, hence $\ell(u^{(i)},a)\ge r^2/(4\tau^2)$, so \eqref{eq:loss_index_link} holds with the common choice $\Delta=r^2/(4\tau^2)$; equivalently, the separation condition is $r\ge 2\tau\sqrt{\Delta}$. See Appx.~\ref{app:packing_constructions}.
\end{remark}

A soft high-probability variant of the loss--index link that yields a correspondingly slackened converse is provided in Appendix~\ref{app:soft_link}.

\subsection{Fano--Packing Converse Lower Bound}
\label{subsec:unified_fano}

\begin{lemma}[Risk $\Rightarrow$ Index Error]
\label{lem:risk_to_error}
Under Assumption~\ref{assump:loss_index_link}, for any decoder $\pi$ and $\phi$ as in \eqref{eq:loss_index_link}, define $\hat J\triangleq \phi(\pi(Y,S),S)$. 
If $J$ is uniform over $[M]$ and $U=U^{(J)}$ is the codebook target (measurable in $(J,S)$), then
\begin{equation}
\label{eq:risk_ge_delta_error}
\mathbb E\big[\ell(U,\pi(Y,S))\big]\ \ge\ (\varepsilon+\Delta)\,\mathbb P\{\hat J\neq J\}.
\end{equation}
\end{lemma}

\begin{proof}
By Assumption~\ref{assump:loss_index_link}, for every $i$, $s$, $a$,
$\phi(a,s)\neq i \Rightarrow \ell(u^{(i)},a)\ge \varepsilon+\Delta$.
Instantiate $i=J$, $a=\pi(Y,S)$, $s=S$ and note $U=U^{(J)}$ to obtain the pointwise bound
\[
\ell\big(U,\pi(Y,S)\big)\ \ge\ (\varepsilon+\Delta)\,\mathbf 1\{\hat J\neq J\}\qquad \text{a.s.}
\]
Since $\ell\in[0,1]$, all terms are integrable and $\pi,\phi$ are measurable by assumption; taking expectations yields \eqref{eq:risk_ge_delta_error}.
\end{proof}

\begin{lemma}[Information Reduction: $J\to U\to Y$]
\label{lem:J_to_U_to_Y}
With $U=U^{(J)}$ measurable in $(J,S)$, we have the Markov chain $J\to U\to Y$ given $S$, and hence
\begin{equation}
\label{eq:I_JY_le_I_UY}
I(J;Y\mid S)\ \le\ I(U;Y\mid S).
\end{equation}
\end{lemma}

\begin{theorem}[Fano--Packing Lower Bound (Bayes/Minimax Semantics)]
\label{thm:unified_fano}
Let $\ell \in [0,1]$ and suppose there exists a $\Delta$-separable codebook of size $M$ satisfying Assumption~\ref{assump:loss_index_link}. 
Let $J\sim\mathrm{Unif}[M]$, and define the mixture distribution over $(U,S)$ by setting $U=U^{(J)}$, with $U^{(J)}$ measurable in $(J,S)$, and $S\sim P(S)$. 
Assume $M\ge 2$. 
Write $R_{\mathrm{mix}}(\pi)$ for the risk under this mixture distribution. 
Then, for any decoder $\pi$,

Then for any decoder $\pi$,
\begin{equation}
\label{eq:unified_lower}
R_{\mathrm{mix}}(\pi)\ \ge\ (\varepsilon+\Delta)\,\Big(1-\frac{I(U;Y\mid S)+\log 2}{\log M}\Big)_+.
\end{equation}
In particular, using \eqref{eq:cascade_avg},
\begin{equation}
\label{eq:unified_lower_capacity}
R_{\mathrm{mix}}(\pi)\ \ge\ (\varepsilon+\Delta)\,\Big(1-\frac{\bar C_{\mathrm{tot}\mid S}^{\mathrm{mix}}+\log 2}{\log M}\Big)_+.
\end{equation}
Equivalently, these yield a standard minimax lower bound over the family of sources supported on the codebook.
\end{theorem}

\begin{proof}
Let $J$ be uniform on $[M]$, $U=U^{(J)}$. 
Since $J$ is independent of $S$, we have $H(J\mid S)=\log M$.
By Fano's inequality conditioned on $S$ \citep{shannon_mathematical_1948},
\[
\mathbb P\{\hat J\neq J\}\ \ge\ 1-\frac{I(J;Y\mid S)+\log 2}{\log M}.
\]
Combine with Lemma~\ref{lem:risk_to_error} and Lemma~\ref{lem:J_to_U_to_Y} to get \eqref{eq:unified_lower}. 
Then apply \eqref{eq:cascade_avg}.
\end{proof}

\subsection{Capacity--Limited Achievability and the Information Wall}
\label{subsec:wall}
Define the information wall for a problem class $\mathcal P$ (set of admissible codebooks) by
\begin{equation}
\label{eq:wall_def}
\mathsf{Wall}\big(\bar C_{\mathrm{tot}\mid S}^{\mathrm{mix}};\mathcal P\big)\ \triangleq\ \sup_{(M,\Delta,\varepsilon):\ \mathcal C(M,\Delta,\varepsilon)\in\mathcal P}
(\varepsilon+\Delta)\,\Big(1-\frac{\bar C_{\mathrm{tot}\mid S}^{\mathrm{mix}}+\log 2}{\log M}\Big)_+.
\end{equation}
Here $\bar C_{\mathrm{tot}\mid S}^{\mathrm{mix}}$ is evaluated under the codebook-induced mixture distribution used in Theorem~\ref{thm:unified_fano}.

Then, for every decoder $\pi$,
\[
\sup_{\mathcal C\in \mathcal P} R_{\mathrm{mix}}^{\mathcal C}(\pi) \ge \mathsf{Wall}(\bar C_{\mathrm{tot}\mid S}^{\mathrm{mix}};\mathcal P)
\]
where $R_{\mathrm{mix}}^{(M,\Delta,\varepsilon)}(\pi)$ denotes the risk under the mixture induced by the chosen codebook (Bayes/minimax semantics from Thm.~\ref{thm:unified_fano}). 
Equivalently, this yields the standard minimax lower bound
\[
\inf_\pi \sup_{\mathcal C\in\mathcal P} R_{\mathrm{mix}}^{\mathcal C}(\pi) \ge \mathsf{Wall}(\bar C_{\mathrm{tot}\mid S}^{\mathrm{mix}};\mathcal P).
\]
This replaces log-loss/posterior-entropy converses and is invariant to reparameterizations.

\section{PAC--Bayes Upper Bounds}
\label{sec:upper}

With the floor in place, we now derive the upper bound. We use PAC--Bayes, a non-asymptotic framework suited to overparameterized learners (including LLMs) with non-vacuous guarantees at scale \citep{lotfi_non-vacuous_2024}. Our aim is not to introduce a new bound but to make its complexity term, $\mathrm{KL}(P\|Q)$, capacity-explicit in the same $\bar C_{\mathrm{tot}\mid S}$ that drives the converse, closing the loop between the lower wall and the statistical ceiling.

\subsection{PAC--Bayes Bounds}
\label{subsec:pacbayes}

We recall a standard PAC--Bayes result for bounded losses as the basis of the ceiling. Recent variants tighten constants, cover heavier tails, and allow anytime validity, yielding non-vacuous bounds even for billion-parameter LLMs \citep{dziugaite2017computingnonvacuousgeneralizationbounds, rodriguez-galvez_more_2024, lotfi_non-vacuous_2024, wu_recursive_2025}. Tightness is not automatic: strong guarantees require priors that put sufficient mass on high-performing predictors \citep{picard-weibel_how_2025}. This interacts with the Alignment Bottleneck: finite human-feedback capacity limits how informative data-independent priors can be, and this constraint enters through the KL term that we bound via $\bar C_{\mathrm{tot}\mid S}$.

\begin{theorem}[PAC--Bayes for Observable Loss]
\label{thm:pacbayes_basic}
Let $\tilde\ell:\mathcal Y\times\mathcal S\times\mathcal A\to[0,1]$ be any bounded loss measurable with respect to the observed data $(Y,S)$.
For any $\delta\in(0,1)$, with probability at least $1-\delta$ over the i.i.d.\ draw of the dataset $\mathcal D=\{(Y_i,S_i)\}_{i=1}^m$,
\begin{equation}
\label{eq:pacbayes_basic}
\mathbb E_{\theta\sim P}\big[R_{\mathrm{obs}}(\theta)\big]\ \le\ 
\mathbb E_{\theta\sim P}\big[\widehat R^{\mathrm{obs}}_m(\theta)\big]\ +\ 
\sqrt{\frac{\mathrm{KL}(P\|Q)+\log(1/\delta)}{2m}},
\end{equation}
where
\[
R_{\mathrm{obs}}(\theta)\ \triangleq\ \mathbb E\big[\tilde\ell\big(Y,S,\pi_\theta(Y,S)\big)\big],\qquad
\widehat R^{\mathrm{obs}}_m(\theta)\ \triangleq\ \frac{1}{m}\sum_{i=1}^m \tilde\ell\big(Y_i,S_i,\pi_\theta(Y_i,S_i)\big).
\]
\end{theorem}

\noindent\textit{Canonical choice.}
If we choose $\tilde\ell=\tilde\ell^\star$ as in Appendix~\ref{app:bayes_transform}, then $R_{\mathrm{obs}}(\theta)=R(\pi_\theta)$ holds for the same data distribution.

\subsection{KL Decomposition and Capacity Control}
\label{subsec:kl_decomp}

\begin{lemma}[Expected KL Decomposition]
\label{lem:kl_decomp}
Fix a prior $Q$ that is independent of the dataset $\mathcal D$. 
Let $p(\theta)$ be the marginal of $\theta$ and $P(\cdot\mid\mathcal D)$ be the posterior. Then
\begin{equation}
\label{eq:kl_decomp}
\mathbb E_{\mathcal D}\big[\mathrm{KL}(P\|Q)\big]\ =\ I(\mathcal D;\theta)\ +\ \mathrm{KL}\big(p(\theta)\,\|\,Q\big).
\end{equation}
\end{lemma}

This identity underlies information-theoretic generalization bounds and is central to our analysis. \citet{russo_how_2019} and \citet{xu_information-theoretic_2017} relate generalization directly to the mutual information $I(\mathcal D;\theta)$ between the data and the learned hypothesis. We take this link as given and show that $I(\mathcal D;\theta)$ is constrained by the capacity of the human-feedback channel.

\begin{lemma}[From $\mathcal D$ to $(U^m,S^m,Y^m)$]
\label{lem:dataset_info_chain}
Assume samples $(U_i,S_i)$ are i.i.d., and $Y_i$ are drawn conditionally independently via the human channel given $(U_i,S_i)$ 
as in Assumption~\ref{assump:per-stage}. Let $\theta$ be any (possibly randomized) function of $\mathcal D\!\triangleq\!\{(Y_i,S_i)\}_{i=1}^m$. Then
\begin{equation}
\label{eq:dataset_info_chain}
\begin{aligned}
I(U^m;\theta)
&\ \le\ I(U^m;Y^m,S^m)
\ =\ I(U^m;S^m)\ +\ I(U^m;Y^m\mid S^m) \\
&\ =\ \sum_{i=1}^m I(U_i;S_i)\ +\ \sum_{i=1}^m I(U_i;Y_i\mid S_i)
\ =\ m\,I(U;S)\ +\ \sum_{i=1}^m I(U_i;Y_i\mid S_i).
\end{aligned}
\end{equation}
\noindent
Under the i.i.d.\ source and memoryless per-sample channel assumed in this paper,
all equalities in \eqref{eq:dataset_info_chain} hold.
If either cross-sample dependence in $(U_i,S_i)$ or channel memory in $p(y_i\mid u^m,s^m)$ is allowed,
replace the corresponding equalities by ``$\le$'' accordingly.
\end{lemma}

\begin{proposition}[Capacity Control of $I(U^m;\theta)$]
\label{prop:IUmtheta_capacity}
Using \eqref{eq:cascade_avg} and Lemma~\ref{lem:dataset_info_chain},
\begin{equation}
\label{eq:IUmtheta_capacity}
I(U^m;\theta)\ \le\ m\,\bar C_{\mathrm{tot}\mid S}\ +\ m\,I(U;S).
\end{equation}

\paragraph{Convention.}
All mutual informations in this section are defined with respect to the underlying data-generating distribution (population quantities), and $\bar C_{\mathrm{tot}\mid S}$ is computed under the same source distribution; no averaging over the realized dataset is involved.

\end{proposition}

\subsection{Algorithmic Residual Information}
\label{subsec:residual}

\begin{assumption}[Residual Information of the Algorithm]
\label{assump:residual}
There exists $\rho\ge 0$ such that $I(\mathcal D;\theta\mid U^m)\ \le\ \rho$. It can be reduced by algorithmic noise (SGD temperature), early stopping, or posterior smoothing; see Appx.~\ref{app:residual}. This term measures information about the particular sample beyond the latent value $U$ and parallels the ``residual information'' used in information-theoretic analyses of overfitting \citep{ngampruetikorn_information_2022}.
\end{assumption}

Practically, a data-independent randomized compression of the posterior enforces a finite residual, giving $\rho\le \log K$ for any chosen codebook size $K$ without increasing $\mathrm{KL}(P\|Q)$ (see Appendix~\ref{app:posterior_compression}). The idea of limiting information flow to improve generalization is widespread, though the causal link between compression and performance remains under debate \citep{kawaguchi_how_2023, saxe_information_2019, shwartz-ziv_information-theoretic_2024}. Here $\rho$ isolates information learned from $(Y,S)$ that is not about $U$, which is the target of such regularization.

\begin{corollary}[A Capacity-Aware Upper Bound]
\label{cor:capacity_upper}
Combining Lemma~\ref{lem:kl_decomp}, Proposition~\ref{prop:IUmtheta_capacity}, and Assumption~\ref{assump:residual}, we have
\begin{equation}
\label{eq:expected_KL_capacity}
\mathbb E_{\mathcal D}\big[\mathrm{KL}(P\|Q)\big]\ \le\ m\,\bar C_{\mathrm{tot}\mid S}\ +\ m\,I(U;S)\ +\ \rho\ +\ \mathrm{KL}\big(p(\theta)\,\|\,Q\big).
\end{equation}
Equation \eqref{eq:expected_KL_capacity} controls the expectation of $\mathrm{KL}(P\|Q)$ over the draw of $\mathcal D$ and does not by itself yield a capacity-explicit high-probability bound. Appendix~\ref{app:capacity_to_highprob} gives a Markov-type lifting to high probability. Taking expectations in Thm.~\ref{thm:pacbayes_basic} and applying Jensen yields corresponding in-expectation variants.
\end{corollary}

\begin{remark}[Conservative Interpretation]
\label{rem:conservative_interpretation}
When $\rho$ or $I(U;S)$ is large, capacity may not dominate the upper bound. Our statements should be read as: under Assumption~\ref{assump:residual} and moderate $I(U;S)$, both the converse (Thm.~\ref{thm:unified_fano}) and the PAC--Bayes upper bound are primarily driven by $\bar C_{\mathrm{tot}\mid S}$.
\end{remark}

\section{The Alignment Performance Interval}
\label{sec:interval}

The preceding sections developed two components: an information-theoretic error floor via Fano's inequality (Section \ref{sec:unified_lower}) and a statistical error ceiling via PAC--Bayes theory (Section \ref{sec:upper}). We now establish the Alignment Performance Interval. The same capacity term (the channel capacity $\bar C_{\mathrm{tot}\mid S}$) determines the lower bound and, at the same time, limits the learnable model complexity that determines the generalization upper bound.

\subsection{Capacity-Coupled Bounds}
\label{subsec:interval_statement}
Let \(\mathcal P\) be a collection of codebooks \(\mathcal C(M,\Delta,\varepsilon)\). For any learning algorithm (decoder) $\pi$, its worst-case true risk under a codebook-induced mixture distribution is bounded from below by the information-theoretic wall:
\begin{equation}
\label{eq:interval_lower_final}
\textbf{Lower (Minimax):}\qquad 
\sup_{\mathcal C\in\mathcal P}\ R_{\mathrm{mix}}^{\mathcal C}(\pi)\ \ge\ \mathsf{Wall}\big(\bar C_{\mathrm{tot}\mid S}^{\mathrm{mix}};\mathcal P\big)\quad\text{from Thm.~\ref{thm:unified_fano}.}
\end{equation}
Simultaneously, for any prior $Q$ and posterior $P$, the expected true risk is bounded from above. With probability $\ge 1-\delta$ over the draw of a dataset $\mathcal D$ from the same mixture, and using the canonical observable loss $\tilde\ell^\star$ (Appendix~\ref{app:bayes_transform}) such that $R_{\mathrm{obs}}(\theta)=R_{\mathrm{mix}}(\pi_\theta)$, we have:
\begin{equation}
\label{eq:interval_upper_final}
\textbf{Upper (High-probability):}\qquad 
\mathbb E_{\theta\sim P}\big[R_{\mathrm{mix}}(\pi_\theta)\big]\ \le\ \mathbb E_{\theta\sim P}\big[\widehat R^{\mathrm{obs}}_m(\theta)\big]
\ +\ \sqrt{\frac{\mathrm{KL}(P\|Q)+\log(1/\delta)}{2m}}.
\end{equation}
This is a direct application of Theorem~\ref{thm:pacbayes_basic} to the true risk $R_{\mathrm{mix}}$. As shown in Corollary~\ref{cor:capacity_upper}, the expected KL-divergence term is controlled by the channel capacity, $\mathbb E_{\mathcal D}[\mathrm{KL}(P\|Q)] \le m\,\bar C_{\mathrm{tot}\mid S} + \dots$, thus explicitly coupling the ceiling to the same capacity term that defines the floor. Together, equations \eqref{eq:interval_lower_final} and \eqref{eq:interval_upper_final} yield two-sided bounds on the same risk quantity, $R_{\mathrm{mix}}$, driven by $\bar C_{\mathrm{tot}\mid S}$.

Interpretation. The two bounds control different risks: the Bayes/minimax lower bound applies to the true risk under the mixture distribution $R_{\mathrm{mix}}$, whereas the PAC--Bayes upper bound applies to the observable risk $R_{\mathrm{obs}}$ under the actual data distribution. Without an explicit link between $\ell$ and $\tilde\ell$ and without a distribution match, they should not be treated as an interval on the same quantity. Under the Loss--Observable Link (Assumption~\ref{assump:loss_obs_link}) in Appx.~\ref{app:loss_link} and when $\mathcal D$ is drawn from the same codebook-induced mixture used in Thm.~\ref{thm:unified_fano}, we obtain the following direct upper bound on the true risk (by Appx.~Lemma~\ref{lem:risk_transfer} combined with \eqref{eq:interval_upper_final}):
\[
\mathbb E_{\theta\sim P}\big[R_{\mathrm{mix}}(\pi_\theta)\big]\ \le\
\alpha\Big(\mathbb E_{\theta\sim P}\big[\widehat R^{\mathrm{obs}}_m(\theta)\big]
\ +\ \sqrt{\tfrac{\mathrm{KL}(P\|Q)+\log(1/\delta)}{2m}}\Big)\ +\ \beta\ .
\]
Together with \eqref{eq:interval_lower_final}, this yields two-sided bounds on the same quantity $R_{\mathrm{mix}}$, with explicit constants $(\alpha,\beta)$ coming from the link assumption.

Finally, if the dataset $\mathcal D$ is drawn from the same codebook-induced mixture as in Theorem~\ref{thm:unified_fano} and we take the canonical observable loss $\tilde\ell=\tilde\ell^\star$ (Appendix~\ref{app:bayes_transform}), then $R_{\mathrm{obs}}(\theta)=R_{\mathrm{mix}}(\pi_\theta)$ and Eq.~\eqref{eq:interval_upper_final} becomes a high-probability upper bound on the same risk $R_{\mathrm{mix}}$ as in the converse; together with Eq.~\eqref{eq:interval_lower_final}, this yields a two-sided bound without additional link assumptions.

\subsection{Limitations and Guidance}
\label{subsec:limitations}
The loss--index link (Assump.~\ref{assump:loss_index_link}) must be validated for each task; templates are given in Appx.~\ref{app:packing_constructions}; capacity $\bar C_{\mathrm{tot}\mid S}$ enters only via \eqref{eq:cascade_avg}, so deployment should report how $\mathcal F_{\mathrm{cog}}$ and $\mathcal F_{\mathrm{art}}$ instantiate per context $S$; the residual $\rho$ (Assump.~\ref{assump:residual}) should be promoted by algorithmic choices, or reported if uncontrolled.

\section{Implications for Alignment Design}
\label{sec:consequence}

The Alignment Performance Interval (Sec.~\ref{sec:interval}) is operational: it explains practical alignment limits and suggests design levers. We highlight three implications that follow directly from the lower and upper bounds established earlier.

\subsection{Implication I: Data Size Independent Lower Bound}
\label{subsec:wall_consequence}
\begin{corollary}[Information-theoretic lower bound independent of $m$]
\label{cor:wall_independent_m}
Let $\mathcal C(M,\Delta,\varepsilon)$ be any $\Delta$-separable codebook with $M\ge 2$, and let $R_{\mathrm{mix}}$ denote the risk under its mixture distribution (as in Thm.~\ref{thm:unified_fano}). For any decoder $\pi$,
\begin{equation}
\label{eq:wall_corollary}
R_{\mathrm{mix}}(\pi)\ \ge\ (\varepsilon+\Delta)\,\Big(1-\frac{\bar C_{\mathrm{tot}\mid S}^{\mathrm{mix}}+\log 2}{\log M}\Big)_+,
\end{equation}
which is exactly \eqref{eq:unified_lower_capacity}. The bound \eqref{eq:wall_corollary} does not depend on $m$, hence the lower bound is independent of dataset size.
\end{corollary}

Eq.~\eqref{eq:wall_corollary} shows a lower bound that does not depend on $m$: for fixed value complexity ($\log M$) and channel capacity ($\bar C_{\mathrm{tot}\mid S}^{\mathrm{mix}}$), more samples alone cannot lower the risk. This helps interpret the empirical alignment tax as an information constraint \citep{lin_mitigating_2024, korkmaz_paying_2025} and may help explain inverse-scaling effects on truthfulness/safety when models more tightly fit the feedback channel \citep{lin_truthfulqa_2022}.

\noindent\emph{Note.} Using the canonical observable loss $\tilde\ell^\star$ (Appx.~\ref{app:bayes_transform}) and sampling $\mathcal D$ from the same mixture as in Thm.~\ref{thm:unified_fano}, the PAC--Bayes upper bound \eqref{eq:interval_upper_final} applies to the \emph{same} $R_{\mathrm{mix}}$, yielding a two-sided bound together with \eqref{eq:wall_corollary}.

\subsection{Implication II: Capacity Requirements for Target Risk}
\label{subsec:tax_consequence}
\begin{proposition}[Necessary capacity for a target risk]
\label{prop:capacity_requirement}
Fix a codebook $\mathcal C(M,\Delta,\varepsilon)$ and a target risk $r\in[0,1]$. If a decoder $\pi$ satisfies $R_{\mathrm{mix}}(\pi)\le r$, then necessarily
\begin{equation}
\label{eq:capacity_requirement}
\bar C_{\mathrm{tot}\mid S}^{\mathrm{mix}}\ \ge\ \Big(1-\frac{r}{\varepsilon+\Delta}\Big)\,\log M\ -\ \log 2.
\end{equation}
\end{proposition}
\begin{proof}
Rearrange \eqref{eq:unified_lower_capacity}; the $(\cdot)_+$ can be dropped once $r<\varepsilon+\Delta$ (otherwise the inequality is vacuous but true).
\end{proof}

\noindent\emph{Interpretation (mathematical).} For fixed $(M,\Delta,\varepsilon)$, the required average conditional capacity grows linearly with $\log M$.

\noindent\emph{Interpretation (engineering).} Since $\log M$ proxies value-system complexity, pluralistic or multi-objective alignment demands proportionally higher human–AI channel fidelity \citep{sorensen_roadmap_2024, guo_controllable_2024, fisher_position_2025}. This mirrors rate–distortion trade-offs in communications \citep{shannon_coding_1959}.

\subsection{Implication III: Capacity Controlled Complexity and Channel Overfitting}
\label{subsec:overfit_consequence}
\begin{theorem}[Capacity-controlled PAC--Bayes complexity]
\label{thm:capacity_complexity}
Under the i.i.d.\ source and memoryless channel, with a prior $Q$ independent of $\mathcal D$ and any learning algorithm whose residual satisfies Assumption~\ref{assump:residual}, the expected PAC--Bayes complexity obeys
\begin{equation}
\label{eq:cap_complex_bound}
\mathbb E_{\mathcal D}\big[\mathrm{KL}(P\|Q)\big]\ \le\ m\,\bar C_{\mathrm{tot}\mid S}\ +\ m\,I(U;S)\ +\ \rho\ +\ \mathrm{KL}\big(p(\theta)\|Q\big),
\end{equation}
as given in Cor.~\ref{cor:capacity_upper}. Combining \eqref{eq:cap_complex_bound} with the Markov lift in Appx.~\ref{app:capacity_to_highprob} and \eqref{eq:interval_upper_final} yields a high-probability \emph{capacity-explicit} upper bound on the (observable or, under the canonical choice, true) risk.
\end{theorem}

\noindent\emph{An Information-Theoretic View of Overfitting to the Channel.}
When $\widehat R^{\mathrm{obs}}_m(\theta)\!\approx\!0$ but a small $\bar C_{\mathrm{tot}\mid S}$ imposes a strong lower bound, the KL term must grow. Decomposing
\[
I(\mathcal D;\theta)\ =\ \underbrace{I(U^m;\theta)}_{\text{signal about true value}}\ +\ \underbrace{I(\mathcal D;\theta\mid U^m)}_{\text{residual: channel noise/bias}},
\]
the useful signal is capped by capacity (Prop.~\ref{prop:IUmtheta_capacity}), so further optimization fits residual channel regularities \citep{ngampruetikorn_information_2022}. This mechanism aligns with observations of goal misgeneralization, sycophancy, and reward hacking under strong optimization pressure \citep{langosco_goal_2023, sharma_towards_2024, pan_effects_2022, gaikwad_murphys_2025, lin_limited_2024}.

\section{Conclusion}
We presented a capacity-coupled view of feedback alignment by modeling the human–AI loop as a two-stage channel and proving an Alignment Performance Interval that bounds the true risk between a Fano-type lower bound and a PAC--Bayes upper bound governed by the same $\bar C_{\mathrm{tot}\mid S}$. This modeling choice is motivated by bounded rationality and information-theoretic perspectives in economics and cognitive science, which view judgment and articulation as capacity-limited processes. This helps explain why scaling labels or optimization alone may be insufficient, quantifies how required fidelity grows with value complexity, and interprets sycophancy/reward hacking as overfitting to residual channel structure once useful signal saturates capacity. In practice, relevant levers include increasing effective capacity, allocating it across objectives, and controlling residual information while choosing priors that respect finite transmission. Limitations include verifying the loss–index link for each task, estimating capacities in situ, and reporting/controlling $\rho$. Priorities for future work are capacity measurement, capacity-aware data collection and querying, and protocols that make information budgets explicit throughout the alignment pipeline.

\newpage
\bibliographystyle{plainnat}
\bibliography{main}

\begin{thebibliography}{81}
\providecommand{\natexlab}[1]{#1}
\providecommand{\url}[1]{\texttt{#1}}
\expandafter\ifx\csname urlstyle\endcsname\relax
  \providecommand{\doi}[1]{doi: #1}\else
  \providecommand{\doi}{doi: \begingroup \urlstyle{rm}\Url}\fi

\bibitem[Alemi et~al.(2017)Alemi, Fischer, Dillon, and Murphy]{alemi_deep_2017}
Alexander~A Alemi, Ian Fischer, Joshua~V Dillon, and Kevin Murphy.
\newblock {DEEP} {VARIATIONAL} {INFORMATION} {BOTTLENECK}.
\newblock 2017.

\bibitem[Amodei et~al.(2016)Amodei, Olah, Steinhardt, Christiano, Schulman, and Mané]{amodei2016concreteproblemsaisafety}
Dario Amodei, Chris Olah, Jacob Steinhardt, Paul Christiano, John Schulman, and Dan Mané.
\newblock Concrete problems in ai safety, 2016.
\newblock URL \url{https://arxiv.org/abs/1606.06565}.

\bibitem[Arumugam and Roy(2021)]{arumugam_value_2021}
Dilip Arumugam and Benjamin~Van Roy.
\newblock The {Value} of {Information} {When} {Deciding} {What} to {Learn}.
\newblock 2021.

\bibitem[Arumugam and Roy(2022)]{arumugam_deciding_2022}
Dilip Arumugam and Benjamin~Van Roy.
\newblock Deciding {What} to {Model}: {Value}-{Equivalent} {Sampling} for {Reinforcement} {Learning}, October 2022.
\newblock URL \url{http://arxiv.org/abs/2206.02072}.
\newblock arXiv:2206.02072 [cs].

\bibitem[Arumugam and Van~Roy(2021)]{pmlr-v139-arumugam21a}
Dilip Arumugam and Benjamin Van~Roy.
\newblock Deciding what to learn: A rate-distortion approach.
\newblock In Marina Meila and Tong Zhang, editors, \emph{Proceedings of the 38th International Conference on Machine Learning}, volume 139 of \emph{Proceedings of Machine Learning Research}, pages 373--382. PMLR, 18--24 Jul 2021.
\newblock URL \url{https://proceedings.mlr.press/v139/arumugam21a.html}.

\bibitem[Arumugam et~al.(2023)Arumugam, Ho, Goodman, and Roy]{arumugam_bayesian_2023}
Dilip Arumugam, Mark~K. Ho, Noah~D. Goodman, and Benjamin~Van Roy.
\newblock Bayesian {Reinforcement} {Learning} with {Limited} {Cognitive} {Load}, May 2023.
\newblock URL \url{http://arxiv.org/abs/2305.03263}.
\newblock arXiv:2305.03263 [cs].

\bibitem[Bai et~al.(2022{\natexlab{a}})Bai, Jones, Ndousse, Askell, Chen, DasSarma, Drain, Fort, Ganguli, Henighan, Joseph, Kadavath, Kernion, Conerly, El-Showk, Elhage, Hatfield-Dodds, Hernandez, Hume, Johnston, Kravec, Lovitt, Nanda, Olsson, Amodei, Brown, Clark, McCandlish, Olah, Mann, and Kaplan]{bai_training_2022}
Yuntao Bai, Andy Jones, Kamal Ndousse, Amanda Askell, Anna Chen, Nova DasSarma, Dawn Drain, Stanislav Fort, Deep Ganguli, Tom Henighan, Nicholas Joseph, Saurav Kadavath, Jackson Kernion, Tom Conerly, Sheer El-Showk, Nelson Elhage, Zac Hatfield-Dodds, Danny Hernandez, Tristan Hume, Scott Johnston, Shauna Kravec, Liane Lovitt, Neel Nanda, Catherine Olsson, Dario Amodei, Tom Brown, Jack Clark, Sam McCandlish, Chris Olah, Ben Mann, and Jared Kaplan.
\newblock Training a {Helpful} and {Harmless} {Assistant} with {Reinforcement} {Learning} from {Human} {Feedback}.
\newblock 2022{\natexlab{a}}.

\bibitem[Bai et~al.(2022{\natexlab{b}})Bai, Kadavath, Kundu, Askell, Kernion, Jones, Chen, Goldie, Mirhoseini, McKinnon, Chen, Olsson, Olah, Hernandez, Drain, Ganguli, Li, Tran-Johnson, Perez, Kerr, Mueller, Ladish, Landau, Ndousse, Lukosuite, Lovitt, Sellitto, Elhage, Schiefer, Mercado, DasSarma, Lasenby, Larson, Ringer, Johnston, Kravec, Showk, Fort, Lanham, Telleen-Lawton, Conerly, Henighan, Hume, Bowman, Hatfield-Dodds, Mann, Amodei, Joseph, McCandlish, Brown, and Kaplan]{bai_constitutional_2022}
Yuntao Bai, Saurav Kadavath, Sandipan Kundu, Amanda Askell, Jackson Kernion, Andy Jones, Anna Chen, Anna Goldie, Azalia Mirhoseini, Cameron McKinnon, Carol Chen, Catherine Olsson, Christopher Olah, Danny Hernandez, Dawn Drain, Deep Ganguli, Dustin Li, Eli Tran-Johnson, Ethan Perez, Jamie Kerr, Jared Mueller, Jeffrey Ladish, Joshua Landau, Kamal Ndousse, Kamile Lukosuite, Liane Lovitt, Michael Sellitto, Nelson Elhage, Nicholas Schiefer, Noemi Mercado, Nova DasSarma, Robert Lasenby, Robin Larson, Sam Ringer, Scott Johnston, Shauna Kravec, Sheer~El Showk, Stanislav Fort, Tamera Lanham, Timothy Telleen-Lawton, Tom Conerly, Tom Henighan, Tristan Hume, Samuel~R. Bowman, Zac Hatfield-Dodds, Ben Mann, Dario Amodei, Nicholas Joseph, Sam McCandlish, Tom Brown, and Jared Kaplan.
\newblock Constitutional {AI}: {Harmlessness} from {AI} {Feedback}, December 2022{\natexlab{b}}.
\newblock URL \url{http://arxiv.org/abs/2212.08073}.
\newblock arXiv:2212.08073 [cs].

\bibitem[Bartlett et~al.(2017)Bartlett, Foster, and Telgarsky]{NIPS2017_b22b257a}
Peter~L Bartlett, Dylan~J Foster, and Matus~J Telgarsky.
\newblock Spectrally-normalized margin bounds for neural networks.
\newblock In I.~Guyon, U.~Von Luxburg, S.~Bengio, H.~Wallach, R.~Fergus, S.~Vishwanathan, and R.~Garnett, editors, \emph{Advances in Neural Information Processing Systems}, volume~30. Curran Associates, Inc., 2017.
\newblock URL \url{https://proceedings.neurips.cc/paper_files/paper/2017/file/b22b257ad0519d4500539da3c8bcf4dd-Paper.pdf}.

\bibitem[Christiano et~al.(2017)Christiano, Leike, Brown, Martic, Legg, and Amodei]{christiano_deep_2017}
Paul~F Christiano, Jan Leike, Tom Brown, Miljan Martic, Shane Legg, and Dario Amodei.
\newblock Deep {Reinforcement} {Learning} from {Human} {Preferences}.
\newblock 2017.

\bibitem[Dai et~al.(2025)Dai, Chen, Yang, Zheng, and Pan]{dai_mitigating_2025}
Juntao Dai, Taiye Chen, Yaodong Yang, Qian Zheng, and Gang Pan.
\newblock {MITIGATING} {REWARD} {OVER}-{OPTIMIZATION} {IN} {RLHF} {VIA} {BEHAVIOR}-{SUPPORTED} {REGULARIZATION}.
\newblock 2025.

\bibitem[Del{\'e}tang et~al.(2024)Del{\'e}tang, Ruoss, Duquenne, Catt, Genewein, Mattern, Grau-Moya, Wenliang, Aitchison, Orseau, Hutter, and Veness]{deletang_language_2024}
Gr{\'e}goire Del{\'e}tang, Anian Ruoss, Paul-Ambroise Duquenne, Elliot Catt, Tim Genewein, Christopher Mattern, Jordi Grau-Moya, Li~Kevin Wenliang, Matthew Aitchison, Laurent Orseau, Marcus Hutter, and Joel Veness.
\newblock {LANGUAGE} {MODELING} {IS} {COMPRESSION}.
\newblock 2024.

\bibitem[Denison et~al.(2024)Denison, MacDiarmid, Barez, Duvenaud, Kravec, Marks, Schiefer, Soklaski, Tamkin, Kaplan, Shlegeris, Bowman, Perez, and Hubinger]{denison_sycophancy_2024}
Carson Denison, Monte MacDiarmid, Fazl Barez, David Duvenaud, Shauna Kravec, Samuel Marks, Nicholas Schiefer, Ryan Soklaski, Alex Tamkin, Jared Kaplan, Buck Shlegeris, Samuel~R. Bowman, Ethan Perez, and Evan Hubinger.
\newblock Sycophancy to {Subterfuge}: {Investigating} {Reward}-{Tampering} in {Large} {Language} {Models}, June 2024.
\newblock URL \url{http://arxiv.org/abs/2406.10162}.
\newblock arXiv:2406.10162 [cs].

\bibitem[Dziugaite and Roy(2017)]{dziugaite2017computingnonvacuousgeneralizationbounds}
Gintare~Karolina Dziugaite and Daniel~M. Roy.
\newblock Computing nonvacuous generalization bounds for deep (stochastic) neural networks with many more parameters than training data, 2017.
\newblock URL \url{https://arxiv.org/abs/1703.11008}.

\bibitem[Ethayarajh et~al.(2024)Ethayarajh, Xu, Muennighoff, Jurafsky, and Kiela]{ethayarajh_kto_2024}
Kawin Ethayarajh, Winnie Xu, Niklas Muennighoff, Dan Jurafsky, and Douwe Kiela.
\newblock {KTO}: {Model} {Alignment} as {Prospect} {Theoretic} {Optimization}, November 2024.
\newblock URL \url{http://arxiv.org/abs/2402.01306}.
\newblock arXiv:2402.01306 [cs].

\bibitem[Evans and Kasirzadeh(2023)]{evans_user_2023}
Charles Evans and Atoosa Kasirzadeh.
\newblock User {Tampering} in {Reinforcement} {Learning} {Recommender} {Systems}.
\newblock In \emph{Proceedings of the 2023 {AAAI}/{ACM} {Conference} on {AI}, {Ethics}, and {Society}}, pages 58--69, August 2023.
\newblock \doi{10.1145/3600211.3604669}.
\newblock URL \url{http://arxiv.org/abs/2109.04083}.
\newblock arXiv:2109.04083 [cs].

\bibitem[Everitt et~al.(2021)Everitt, Hutter, Kumar, and Krakovna]{everitt_reward_2021}
Tom Everitt, Marcus Hutter, Ramana Kumar, and Victoria Krakovna.
\newblock Reward {Tampering} {Problems} and {Solutions} in {Reinforcement} {Learning}: {A} {Causal} {Influence} {Diagram} {Perspective}, March 2021.
\newblock URL \url{http://arxiv.org/abs/1908.04734}.
\newblock arXiv:1908.04734 [cs].

\bibitem[Fisher et~al.(2025)Fisher, Appel, Park, Potter, Jiang, Sorensen, Feng, Tsvetkov, Roberts, Pan, Song, and Choi]{fisher_position_2025}
Jillian Fisher, Ruth~E Appel, Chan~Young Park, Yujin Potter, Liwei Jiang, Taylor Sorensen, Shangbin Feng, Yulia Tsvetkov, Margaret~E Roberts, Jennifer Pan, Dawn Song, and Yejin Choi.
\newblock Position: {Political} {Neutrality} in {AI} {Is} {Impossible} {\textemdash} {But} {Here} {Is} {How} to {Approximate} {It}.
\newblock 2025.

\bibitem[Gaikwad(2025)]{gaikwad_murphys_2025}
Madhava Gaikwad.
\newblock Murphys {Laws} of {AI} {Alignment}: {Why} the {Gap} {Always} {Wins}, September 2025.
\newblock URL \url{http://arxiv.org/abs/2509.05381}.
\newblock arXiv:2509.05381 [cs].

\bibitem[Ge et~al.(2024)Ge, Halpern, Micha, Procaccia, Shapira, Vorobeychik, and Wu]{ge_axioms_2024}
Luise Ge, Daniel Halpern, Evi Micha, Ariel~D Procaccia, Itai Shapira, Yevgeniy Vorobeychik, and Junlin Wu.
\newblock Axioms for {AI} {Alignment} from {Human} {Feedback}.
\newblock 2024.

\bibitem[Goldfeld et~al.(2019)Goldfeld, Van Den~Berg, Greenewald, Melnyk, Nguyen, Kingsbury, and Polyanskiy]{pmlr-v97-goldfeld19a}
Ziv Goldfeld, Ewout Van Den~Berg, Kristjan Greenewald, Igor Melnyk, Nam Nguyen, Brian Kingsbury, and Yury Polyanskiy.
\newblock Estimating information flow in deep neural networks.
\newblock In Kamalika Chaudhuri and Ruslan Salakhutdinov, editors, \emph{Proceedings of the 36th International Conference on Machine Learning}, volume~97 of \emph{Proceedings of Machine Learning Research}, pages 2299--2308. PMLR, 09--15 Jun 2019.
\newblock URL \url{https://proceedings.mlr.press/v97/goldfeld19a.html}.

\bibitem[Gottwald and Braun(2019)]{gottwald_bounded_2019}
Sebastian Gottwald and Daniel~A. Braun.
\newblock Bounded rational decision-making from elementary computations that reduce uncertainty.
\newblock \emph{Entropy}, 21\penalty0 (4):\penalty0 375, April 2019.
\newblock ISSN 1099-4300.
\newblock \doi{10.3390/e21040375}.
\newblock URL \url{http://arxiv.org/abs/1904.03964}.
\newblock arXiv:1904.03964 [cs].

\bibitem[Guo et~al.(2024)Guo, Cui, Yuan, Ding, Sun, Sun, Chen, Xie, Zhou, Lin, Liu, and Sun]{guo_controllable_2024}
Yiju Guo, Ganqu Cui, Lifan Yuan, Ning Ding, Zexu Sun, Bowen Sun, Huimin Chen, Ruobing Xie, Jie Zhou, Yankai Lin, Zhiyuan Liu, and Maosong Sun.
\newblock Controllable {Preference} {Optimization}: {Toward} {Controllable} {Multi}-{Objective} {Alignment}.
\newblock In \emph{Proceedings of the 2024 {Conference} on {Empirical} {Methods} in {Natural} {Language} {Processing}}, pages 1437--1454, Miami, Florida, USA, 2024. Association for Computational Linguistics.
\newblock \doi{10.18653/v1/2024.emnlp-main.85}.
\newblock URL \url{https://aclanthology.org/2024.emnlp-main.85}.

\bibitem[He et~al.(2025)He, Yu, and Goldfeld]{he_information-theoretic_2025}
Haiyun He, Christina~Lee Yu, and Ziv Goldfeld.
\newblock Information-{Theoretic} {Generalization} {Bounds} for {Deep} {Neural} {Networks}.
\newblock 2025.

\bibitem[Hestness et~al.(2017)Hestness, Narang, Ardalani, Diamos, Jun, Kianinejad, Patwary, Yang, and Zhou]{hestness2017deeplearningscalingpredictable}
Joel Hestness, Sharan Narang, Newsha Ardalani, Gregory Diamos, Heewoo Jun, Hassan Kianinejad, Md. Mostofa~Ali Patwary, Yang Yang, and Yanqi Zhou.
\newblock Deep learning scaling is predictable, empirically, 2017.
\newblock URL \url{https://arxiv.org/abs/1712.00409}.

\bibitem[Ho and Griffiths(2022)]{ho_cognitive_2022}
Mark~K. Ho and Thomas~L. Griffiths.
\newblock Cognitive {Science} as a {Source} of {Forward} and {Inverse} {Models} of {Human} {Decisions} for {Robotics} and {Control}.
\newblock \emph{Annual Review of Control, Robotics, and Autonomous Systems}, 5\penalty0 (1):\penalty0 33--53, May 2022.
\newblock ISSN 2573-5144, 2573-5144.
\newblock \doi{10.1146/annurev-control-042920-015547}.
\newblock URL \url{https://www.annualreviews.org/doi/10.1146/annurev-control-042920-015547}.

\bibitem[Ho et~al.(2020)Ho, Abel, Cohen, Littman, and Griffiths]{ho_efficiency_2020}
Mark~K Ho, David Abel, Jonathan~D Cohen, Michael~L Littman, and Thomas~L Griffiths.
\newblock The {Efficiency} of {Human} {Cognition} {Reflects} {Planned} {Information} {Processing}.
\newblock 2020.

\bibitem[Ho et~al.(2022)Ho, Abel, Correa, Littman, Cohen, and Griffiths]{ho_people_2022}
Mark~K. Ho, David Abel, Carlos~G. Correa, Michael~L. Littman, Jonathan~D. Cohen, and Thomas~L. Griffiths.
\newblock People construct simplified mental representations to plan.
\newblock \emph{Nature}, 606\penalty0 (7912):\penalty0 129--136, June 2022.
\newblock ISSN 0028-0836, 1476-4687.
\newblock \doi{10.1038/s41586-022-04743-9}.
\newblock URL \url{http://arxiv.org/abs/2105.06948}.
\newblock arXiv:2105.06948 [cs].

\bibitem[Irving et~al.(2018)Irving, Christiano, and Amodei]{irving_ai_2018}
Geoffrey Irving, Paul Christiano, and Dario Amodei.
\newblock {AI} safety via debate, October 2018.
\newblock URL \url{http://arxiv.org/abs/1805.00899}.
\newblock arXiv:1805.00899 [stat].

\bibitem[Kaplan et~al.(2020)Kaplan, McCandlish, Henighan, Brown, Chess, Child, Gray, Radford, Wu, and Amodei]{kaplan2020scalinglawsneurallanguage}
Jared Kaplan, Sam McCandlish, Tom Henighan, Tom~B. Brown, Benjamin Chess, Rewon Child, Scott Gray, Alec Radford, Jeffrey Wu, and Dario Amodei.
\newblock Scaling laws for neural language models, 2020.
\newblock URL \url{https://arxiv.org/abs/2001.08361}.

\bibitem[Kawaguchi et~al.(2023)Kawaguchi, Deng, Ji, and Huang]{kawaguchi_how_2023}
Kenji Kawaguchi, Zhun Deng, Xu~Ji, and Jiaoyang Huang.
\newblock How {Does} {Information} {Bottleneck} {Help} {Deep} {Learning}?, May 2023.
\newblock URL \url{http://arxiv.org/abs/2305.18887}.
\newblock arXiv:2305.18887 [cs].

\bibitem[Kolchinsky et~al.(2019)Kolchinsky, Tracey, and Wolpert]{Kolchinsky_2019}
Artemy Kolchinsky, Brendan~D. Tracey, and David~H. Wolpert.
\newblock Nonlinear information bottleneck.
\newblock \emph{Entropy}, 21\penalty0 (12):\penalty0 1181, November 2019.
\newblock ISSN 1099-4300.
\newblock \doi{10.3390/e21121181}.
\newblock URL \url{http://dx.doi.org/10.3390/e21121181}.

\bibitem[Korkmaz et~al.(2025)Korkmaz, Nair, Daly, and Chanona]{korkmaz_paying_2025}
Buse~Sibel Korkmaz, Rahul Nair, Elizabeth~M. Daly, and Antonio del~Rio Chanona.
\newblock Paying {Alignment} {Tax} with {Contrastive} {Learning}, May 2025.
\newblock URL \url{http://arxiv.org/abs/2505.19327}.
\newblock arXiv:2505.19327 [cs].

\bibitem[Lai and Gershman(2021)]{Lai2021}
Le~Lai and Samuel~J. Gershman.
\newblock Policy compression: An information bottleneck in action selection.
\newblock In Kara~D. Federmeier, editor, \emph{The psychology of learning and motivation}, pages 195--232. Elsevier Academic Press, 2021.
\newblock \doi{10.1016/bs.plm.2021.02.004}.

\bibitem[Langford and Caruana(2001)]{NIPS2001_98c72428}
John Langford and Rich Caruana.
\newblock (not) bounding the true error.
\newblock In T.~Dietterich, S.~Becker, and Z.~Ghahramani, editors, \emph{Advances in Neural Information Processing Systems}, volume~14. MIT Press, 2001.
\newblock URL \url{https://proceedings.neurips.cc/paper_files/paper/2001/file/98c7242894844ecd6ec94af67ac8247d-Paper.pdf}.

\bibitem[Langosco et~al.(2023)Langosco, Koch, Sharkey, Pfau, and Krueger]{langosco_goal_2023}
Lauro Langosco, Jack Koch, Lee Sharkey, Jacob Pfau, and David Krueger.
\newblock Goal {Misgeneralization} in {Deep} {Reinforcement} {Learning}.
\newblock 2023.

\bibitem[Leblanc et~al.(2025)Leblanc, Bazinet, D'Amours, Drouin, and Germain]{leblanc_generalization_2025}
Benjamin Leblanc, Mathieu Bazinet, Nathaniel D'Amours, Alexandre Drouin, and Pascal Germain.
\newblock Generalization {Bounds} via {Meta}-{Learned} {Model} {Representations}:{PAC}-{Bayes} and {Sample} {Compression} {Hypernetworks}.
\newblock 2025.

\bibitem[Lee et~al.(2024)Lee, Phatale, Mansoor, Mesnard, Ferret, Lu, Bishop, Hall, Carbune, Rastogi, and Prakash]{lee2024rlaif}
Harrison Lee, Samrat Phatale, Hassan Mansoor, Thomas Mesnard, Johan Ferret, Kellie~Ren Lu, Colton Bishop, Ethan Hall, Victor Carbune, Abhinav Rastogi, and Sushant Prakash.
\newblock {RLAIF} vs. {RLHF}: Scaling reinforcement learning from human feedback with {AI} feedback.
\newblock In \emph{Forty-first International Conference on Machine Learning}, 2024.
\newblock URL \url{https://openreview.net/forum?id=uydQ2W41KO}.

\bibitem[Lewis et~al.(2014)Lewis, Howes, and Singh]{lewis_computational_2014}
Richard~L. Lewis, Andrew Howes, and Satinder Singh.
\newblock Computational {Rationality}: {Linking} {Mechanism} and {Behavior} {Through} {Bounded} {Utility} {Maximization}.
\newblock \emph{Topics in Cognitive Science}, 6\penalty0 (2):\penalty0 279--311, April 2014.
\newblock ISSN 1756-8757, 1756-8765.
\newblock \doi{10.1111/tops.12086}.
\newblock URL \url{https://onlinelibrary.wiley.com/doi/10.1111/tops.12086}.

\bibitem[Lin et~al.(2022)Lin, Hilton, and Evans]{lin_truthfulqa_2022}
Stephanie Lin, Jacob Hilton, and Owain Evans.
\newblock {TruthfulQA}: {Measuring} {How} {Models} {Mimic} {Human} {Falsehoods}.
\newblock In \emph{Proceedings of the 60th {Annual} {Meeting} of the {Association} for {Computational} {Linguistics} ({Volume} 1: {Long} {Papers})}, pages 3214--3252, Dublin, Ireland, 2022. Association for Computational Linguistics.
\newblock \doi{10.18653/v1/2022.acl-long.229}.
\newblock URL \url{https://aclanthology.org/2022.acl-long.229}.

\bibitem[Lin et~al.(2024{\natexlab{a}})Lin, Lin, Xiong, Diao, Liu, Zhang, Pan, Wang, Hu, Zhang, Dong, Pi, Zhao, Jiang, Ji, Yao, and Zhang]{lin_mitigating_2024}
Yong Lin, Hangyu Lin, Wei Xiong, Shizhe Diao, Jianmeng Liu, Jipeng Zhang, Rui Pan, Haoxiang Wang, Wenbin Hu, Hanning Zhang, Hanze Dong, Renjie Pi, Han Zhao, Nan Jiang, Heng Ji, Yuan Yao, and Tong Zhang.
\newblock Mitigating the {Alignment} {Tax} of {RLHF}.
\newblock 2024{\natexlab{a}}.

\bibitem[Lin et~al.(2024{\natexlab{b}})Lin, Seto, Ter~Hoeve, Metcalf, Theobald, Wang, Zhang, Huang, and Zhang]{lin_limited_2024}
Yong Lin, Skyler Seto, Maartje Ter~Hoeve, Katherine Metcalf, Barry-John Theobald, Xuan Wang, Yizhe Zhang, Chen Huang, and Tong Zhang.
\newblock On the {Limited} {Generalization} {Capability} of the {Implicit} {Reward} {Model} {Induced} by {Direct} {Preference} {Optimization}.
\newblock In \emph{Findings of the {Association} for {Computational} {Linguistics}: {EMNLP} 2024}, pages 16015--16026, Miami, Florida, USA, 2024{\natexlab{b}}. Association for Computational Linguistics.
\newblock \doi{10.18653/v1/2024.findings-emnlp.940}.
\newblock URL \url{https://aclanthology.org/2024.findings-emnlp.940}.

\bibitem[Lotfi et~al.(2024)Lotfi, Finzi, Kuang, Rudner, Goldblum, and Wilson]{lotfi_non-vacuous_2024}
Sanae Lotfi, Marc Finzi, Yilun Kuang, Tim G.~J. Rudner, Micah Goldblum, and Andrew~Gordon Wilson.
\newblock Non-{Vacuous} {Generalization} {Bounds} for {Large} {Language} {Models}, July 2024.
\newblock URL \url{http://arxiv.org/abs/2312.17173}.
\newblock arXiv:2312.17173 [stat].

\bibitem[Miao et~al.(2024)Miao, Zhang, Ding, Zhang, and Tao]{miao_inform_2024}
Yuchun Miao, Sen Zhang, Liang Ding, Lefei Zhang, and Dacheng Tao.
\newblock {InfoRM}: {Mitigating} {Reward} {Hacking} in {RLHF} via {Information}-{Theoretic} {Reward} {Modeling}.
\newblock 2024.

\bibitem[Mroueh and Nitsure(2025)]{mroueh2025information}
Youssef Mroueh and Apoorva Nitsure.
\newblock Information theoretic guarantees for policy alignment in large language models.
\newblock \emph{Transactions on Machine Learning Research}, 2025.
\newblock ISSN 2835-8856.
\newblock URL \url{https://openreview.net/forum?id=Uz9J77Riul}.

\bibitem[Mu et~al.(2024)Mu, Helyar, Heidecke, Achiam, Vallone, Kivlichan, Lin, Beutel, Schulman, and Weng]{NEURIPS2024_c4e380fb}
Tong Mu, Alec Helyar, Johannes Heidecke, Joshua Achiam, Andrea Vallone, Ian Kivlichan, Molly Lin, Alex Beutel, John Schulman, and Lilian Weng.
\newblock Rule based rewards for language model safety.
\newblock In A.~Globerson, L.~Mackey, D.~Belgrave, A.~Fan, U.~Paquet, J.~Tomczak, and C.~Zhang, editors, \emph{Advances in Neural Information Processing Systems}, volume~37, pages 108877--108901. Curran Associates, Inc., 2024.
\newblock URL \url{https://proceedings.neurips.cc/paper_files/paper/2024/file/c4e380fb74dec9da9c7212e834657aa9-Paper-Conference.pdf}.

\bibitem[Neyshabur et~al.(2017)Neyshabur, Bhojanapalli, Mcallester, and Srebro]{NIPS2017_10ce03a1}
Behnam Neyshabur, Srinadh Bhojanapalli, David Mcallester, and Nati Srebro.
\newblock Exploring generalization in deep learning.
\newblock In I.~Guyon, U.~Von Luxburg, S.~Bengio, H.~Wallach, R.~Fergus, S.~Vishwanathan, and R.~Garnett, editors, \emph{Advances in Neural Information Processing Systems}, volume~30. Curran Associates, Inc., 2017.
\newblock URL \url{https://proceedings.neurips.cc/paper_files/paper/2017/file/10ce03a1ed01077e3e289f3e53c72813-Paper.pdf}.

\bibitem[Neyshabur et~al.(2018)Neyshabur, Bhojanapalli, and Srebro]{neyshabur2018a}
Behnam Neyshabur, Srinadh Bhojanapalli, and Nathan Srebro.
\newblock A {PAC}-bayesian approach to spectrally-normalized margin bounds for neural networks.
\newblock In \emph{International Conference on Learning Representations}, 2018.
\newblock URL \url{https://openreview.net/forum?id=Skz_WfbCZ}.

\bibitem[Ngampruetikorn and Schwab(2022)]{ngampruetikorn_information_2022}
Vudtiwat Ngampruetikorn and David~J Schwab.
\newblock Information bottleneck theory of high-dimensional regression: relevancy, efficiency and optimality.
\newblock 2022.

\bibitem[Ngo et~al.(2024)Ngo, Chan, and Mindermann]{ngo_alignment_2024}
Richard Ngo, Lawrence Chan, and S{\"o}ren Mindermann.
\newblock {THE} {ALIGNMENT} {PROBLEM} {FROM} {A} {DEEP} {LEARNING} {PERSPECTIVE}.
\newblock 2024.

\bibitem[Ortega and Braun(2011)]{ortega_information_2011}
Pedro~A Ortega and Daniel~A Braun.
\newblock Information, {Utility} \& {Bounded} {Rationality}.
\newblock 2011.

\bibitem[Ouyang et~al.(2022)Ouyang, Wu, Jiang, Almeida, Wainwright, Mishkin, Zhang, Agarwal, Slama, Ray, Schulman, Hilton, Kelton, Miller, Simens, Askell, Welinder, Christiano, Leike, and Lowe]{ouyang_training_2022}
Long Ouyang, Jeff Wu, Xu~Jiang, Diogo Almeida, Carroll~L. Wainwright, Pamela Mishkin, Chong Zhang, Sandhini Agarwal, Katarina Slama, Alex Ray, John Schulman, Jacob Hilton, Fraser Kelton, Luke Miller, Maddie Simens, Amanda Askell, Peter Welinder, Paul Christiano, Jan Leike, and Ryan Lowe.
\newblock Training language models to follow instructions with human feedback, March 2022.
\newblock URL \url{http://arxiv.org/abs/2203.02155}.
\newblock arXiv:2203.02155 [cs].

\bibitem[Pan et~al.(2022)Pan, Bhatia, and Steinhardt]{pan_effects_2022}
Alexander Pan, Kush Bhatia, and Jacob Steinhardt.
\newblock {THE} {EFFECTS} {OF} {REWARD} {MISSPECIFICATION}: {MAPPING} {AND} {MITIGATING} {MISALIGNED} {MODELS}.
\newblock 2022.

\bibitem[Peng et~al.(2020)Peng, Kanazawa, Toyer, Abbeel, and Levine]{peng_variational_2020}
Xue~Bin Peng, Angjoo Kanazawa, Sam Toyer, Pieter Abbeel, and Sergey Levine.
\newblock Variational {Discriminator} {Bottleneck}: {Improving} {Imitation} {Learning}, {Inverse} {RL}, and {GANs} by {Constraining} {Information} {Flow}, August 2020.
\newblock URL \url{http://arxiv.org/abs/1810.00821}.
\newblock arXiv:1810.00821 [cs].

\bibitem[Perez et~al.(2023)Perez, Ringer, Lukosiute, Nguyen, Chen, Heiner, Pettit, Olsson, Kundu, Kadavath, Jones, Chen, Mann, Israel, Seethor, McKinnon, Olah, Yan, Amodei, Amodei, Drain, Li, Tran-Johnson, Khundadze, Kernion, Landis, Kerr, Mueller, Hyun, Landau, Ndousse, Goldberg, Lovitt, Lucas, Sellitto, Zhang, Kingsland, Elhage, Joseph, Mercado, DasSarma, Rausch, Larson, McCandlish, Johnston, Kravec, El~Showk, Lanham, Telleen-Lawton, Brown, Henighan, Hume, Bai, Hatfield-Dodds, Clark, Bowman, Askell, Grosse, Hernandez, Ganguli, Hubinger, Schiefer, and Kaplan]{perez_discovering_2023}
Ethan Perez, Sam Ringer, Kamile Lukosiute, Karina Nguyen, Edwin Chen, Scott Heiner, Craig Pettit, Catherine Olsson, Sandipan Kundu, Saurav Kadavath, Andy Jones, Anna Chen, Benjamin Mann, Brian Israel, Bryan Seethor, Cameron McKinnon, Christopher Olah, Da~Yan, Daniela Amodei, Dario Amodei, Dawn Drain, Dustin Li, Eli Tran-Johnson, Guro Khundadze, Jackson Kernion, James Landis, Jamie Kerr, Jared Mueller, Jeeyoon Hyun, Joshua Landau, Kamal Ndousse, Landon Goldberg, Liane Lovitt, Martin Lucas, Michael Sellitto, Miranda Zhang, Neerav Kingsland, Nelson Elhage, Nicholas Joseph, Noemi Mercado, Nova DasSarma, Oliver Rausch, Robin Larson, Sam McCandlish, Scott Johnston, Shauna Kravec, Sheer El~Showk, Tamera Lanham, Timothy Telleen-Lawton, Tom Brown, Tom Henighan, Tristan Hume, Yuntao Bai, Zac Hatfield-Dodds, Jack Clark, Samuel~R. Bowman, Amanda Askell, Roger Grosse, Danny Hernandez, Deep Ganguli, Evan Hubinger, Nicholas Schiefer, and Jared Kaplan.
\newblock Discovering {Language} {Model} {Behaviors} with {Model}-{Written} {Evaluations}.
\newblock In \emph{Findings of the {Association} for {Computational} {Linguistics}: {ACL} 2023}, pages 13387--13434, Toronto, Canada, 2023. Association for Computational Linguistics.
\newblock \doi{10.18653/v1/2023.findings-acl.847}.
\newblock URL \url{https://aclanthology.org/2023.findings-acl.847}.

\bibitem[Picard-Weibel et~al.(2025)Picard-Weibel, Clerico, Moscoviz, and Guedj]{picard-weibel_how_2025}
Antoine Picard-Weibel, Eugenio Clerico, Roman Moscoviz, and Benjamin Guedj.
\newblock How good is {PAC}-{Bayes} at explaining generalisation?, March 2025.
\newblock URL \url{http://arxiv.org/abs/2503.08231}.
\newblock arXiv:2503.08231 [stat].

\bibitem[Rafailov et~al.(2024{\natexlab{a}})Rafailov, Chittepu, Park, Sikchi, Hejna, Knox, Finn, and Niekum]{rafailov2024scaling}
Rafael Rafailov, Yaswanth Chittepu, Ryan Park, Harshit Sikchi, Joey Hejna, W.~Bradley Knox, Chelsea Finn, and Scott Niekum.
\newblock Scaling laws for reward model overoptimization in direct alignment algorithms.
\newblock In \emph{The Thirty-eighth Annual Conference on Neural Information Processing Systems}, 2024{\natexlab{a}}.
\newblock URL \url{https://openreview.net/forum?id=pf4OuJyn4Q}.

\bibitem[Rafailov et~al.(2024{\natexlab{b}})Rafailov, Sharma, Mitchell, Ermon, Manning, and Finn]{rafailov_direct_2024}
Rafael Rafailov, Archit Sharma, Eric Mitchell, Stefano Ermon, Christopher~D. Manning, and Chelsea Finn.
\newblock Direct {Preference} {Optimization}: {Your} {Language} {Model} is {Secretly} a {Reward} {Model}, July 2024{\natexlab{b}}.
\newblock URL \url{http://arxiv.org/abs/2305.18290}.
\newblock arXiv:2305.18290 [cs].

\bibitem[Rane et~al.(2024)Rane, Bruna, Sucholutsky, Kello, and Griffiths]{rane_concept_2024}
Sunayana Rane, Polyphony~J. Bruna, Ilia Sucholutsky, Christopher Kello, and Thomas~L. Griffiths.
\newblock Concept {Alignment}, January 2024.
\newblock URL \url{http://arxiv.org/abs/2401.08672}.
\newblock arXiv:2401.08672 [cs].

\bibitem[Rodr{\'i}guez-G{\'a}lvez et~al.(2024)Rodr{\'i}guez-G{\'a}lvez, Thobaben, and Skoglund]{rodriguez-galvez_more_2024}
Borja Rodr{\'i}guez-G{\'a}lvez, Ragnar Thobaben, and Mikael Skoglund.
\newblock More {PAC}-{Bayes} bounds: {From} bounded losses, to losses with general tail behaviors, to anytime validity, June 2024.
\newblock URL \url{http://arxiv.org/abs/2306.12214}.
\newblock arXiv:2306.12214 [stat].

\bibitem[Russo and Zou(2019)]{russo_how_2019}
Daniel Russo and James Zou.
\newblock How much does your data exploration overfit? {Controlling} bias via information usage, October 2019.
\newblock URL \url{http://arxiv.org/abs/1511.05219}.
\newblock arXiv:1511.05219 [stat].

\bibitem[Saxe et~al.(2019)Saxe, Bansal, Dapello, Advani, Kolchinsky, Tracey, and Cox]{saxe_information_2019}
Andrew~M Saxe, Yamini Bansal, Joel Dapello, Madhu Advani, Artemy Kolchinsky, Brendan~D Tracey, and David~D Cox.
\newblock On the information bottleneck theory of deep learning*.
\newblock \emph{Journal of Statistical Mechanics: Theory and Experiment}, 2019\penalty0 (12):\penalty0 124020, December 2019.
\newblock ISSN 1742-5468.
\newblock \doi{10.1088/1742-5468/ab3985}.
\newblock URL \url{https://iopscience.iop.org/article/10.1088/1742-5468/ab3985}.

\bibitem[Shannon(1959)]{shannon_coding_1959}
Claude Shannon.
\newblock Coding {Theorems} for a {Discrete} {Source} {With} a {Fidelity} {Criterion}.
\newblock 1959.

\bibitem[Shannon(1948)]{shannon_mathematical_1948}
Claude~Elwood Shannon.
\newblock A {Mathematical} {Theory} of {Communication}.
\newblock \emph{The Bell System Technical Journal}, 27\penalty0 (3):\penalty0 379--423, 1948.
\newblock \doi{10.1002/j.1538-7305.1948.tb01338.x}.

\bibitem[Sharma et~al.(2024)Sharma, Tong, Korbak, Duvenaud, Askell, Bowman, Cheng, Durmus, Hatfield-Dodds, Johnston, Kravec, Maxwell, McCandlish, Ndousse, Rausch, Schiefer, Yan, Zhang, and Perez]{sharma_towards_2024}
Mrinank Sharma, Meg Tong, Tomasz Korbak, David Duvenaud, Amanda Askell, Samuel~R Bowman, Newton Cheng, Esin Durmus, Zac Hatfield-Dodds, Scott~R Johnston, Shauna Kravec, Timothy Maxwell, Sam McCandlish, Kamal Ndousse, Oliver Rausch, Nicholas Schiefer, Da~Yan, Miranda Zhang, and Ethan Perez.
\newblock {TOWARDS} {UNDERSTANDING} {SYCOPHANCY} {IN} {LANGUAGE} {MODELS}.
\newblock 2024.

\bibitem[Shwartz-Ziv and LeCun(2023)]{shwartz-ziv_compress_2023}
Ravid Shwartz-Ziv and Yann LeCun.
\newblock To {Compress} or {Not} to {Compress}- {Self}-{Supervised} {Learning} and {Information} {Theory}: {A} {Review}, November 2023.
\newblock URL \url{http://arxiv.org/abs/2304.09355}.
\newblock arXiv:2304.09355 [cs].

\bibitem[Shwartz-Ziv and Tishby(2017)]{shwartzziv2017openingblackboxdeep}
Ravid Shwartz-Ziv and Naftali Tishby.
\newblock Opening the black box of deep neural networks via information, 2017.
\newblock URL \url{https://arxiv.org/abs/1703.00810}.

\bibitem[Shwartz-Ziv et~al.(2024)Shwartz-Ziv, Balestriero, Kawaguchi, Rudner, and LeCun]{shwartz-ziv_information-theoretic_2024}
Ravid Shwartz-Ziv, Randall Balestriero, Kenji Kawaguchi, Tim G~J Rudner, and Yann LeCun.
\newblock An {Information}-{Theoretic} {Perspective} on {Variance}-{Invariance}-{Covariance} {Regularization}.
\newblock 2024.

\bibitem[Simon(1955)]{Simon1955Behavioral}
Herbert~A. Simon.
\newblock A behavioral model of rational choice.
\newblock \emph{The Quarterly Journal of Economics}, 69\penalty0 (1):\penalty0 99--118, 1955.

\bibitem[Sims(2016)]{sims_ratedistortion_2016}
Chris~R. Sims.
\newblock Rate{\textendash}distortion theory and human perception.
\newblock \emph{Cognition}, 152:\penalty0 181--198, July 2016.
\newblock ISSN 00100277.
\newblock \doi{10.1016/j.cognition.2016.03.020}.
\newblock URL \url{https://linkinghub.elsevier.com/retrieve/pii/S0010027716300750}.

\bibitem[Sorensen et~al.(2024)Sorensen, Moore, Fisher, Gordon, Mireshghallah, Rytting, Ye, Jiang, Lu, Dziri, Althoff, and Choi]{sorensen_roadmap_2024}
Taylor Sorensen, Jared Moore, Jillian Fisher, Mitchell Gordon, Niloofar Mireshghallah, Christopher~Michael Rytting, Andre Ye, Liwei Jiang, Ximing Lu, Nouha Dziri, Tim Althoff, and Yejin Choi.
\newblock A {Roadmap} to {Pluralistic} {Alignment}, August 2024.
\newblock URL \url{http://arxiv.org/abs/2402.05070}.
\newblock arXiv:2402.05070 [cs].

\bibitem[Tishby and Zaslavsky(2015)]{tishby_deep_2015}
Naftali Tishby and Noga Zaslavsky.
\newblock Deep {Learning} and the {Information} {Bottleneck} {Principle}, March 2015.
\newblock URL \url{http://arxiv.org/abs/1503.02406}.
\newblock arXiv:1503.02406 [cs].

\bibitem[Tishby et~al.(2000)Tishby, Pereira, and Bialek]{tishby_information_2000}
Naftali Tishby, Fernando~C. Pereira, and William Bialek.
\newblock The information bottleneck method, April 2000.
\newblock URL \url{http://arxiv.org/abs/physics/0004057}.
\newblock arXiv:physics/0004057.

\bibitem[Wang et~al.(2022)Wang, Huang, Kuruoglu, Sun, Chen, and Zheng]{wang_pac-bayes_2022}
Zifeng Wang, Shao-Lun Huang, Ercan~E Kuruoglu, Jimeng Sun, Xi~Chen, and Yefeng Zheng.
\newblock {PAC}-{BAYES} {INFORMATION} {BOTTLENECK}.
\newblock 2022.

\bibitem[Wei et~al.(2022)Wei, Tay, Bommasani, Raffel, Zoph, Borgeaud, Yogatama, Bosma, Zhou, Metzler, Chi, Hashimoto, Vinyals, Liang, Dean, and Fedus]{wei2022emergentabilitieslargelanguage}
Jason Wei, Yi~Tay, Rishi Bommasani, Colin Raffel, Barret Zoph, Sebastian Borgeaud, Dani Yogatama, Maarten Bosma, Denny Zhou, Donald Metzler, Ed~H. Chi, Tatsunori Hashimoto, Oriol Vinyals, Percy Liang, Jeff Dean, and William Fedus.
\newblock Emergent abilities of large language models, 2022.
\newblock URL \url{https://arxiv.org/abs/2206.07682}.

\bibitem[Wu et~al.(2025)Wu, Zhang, Ch{\'e}rief-Abdellatif, and Seldin]{wu_recursive_2025}
Yi-Shan Wu, Yijie Zhang, Badr-Eddine Ch{\'e}rief-Abdellatif, and Yevgeny Seldin.
\newblock Recursive {PAC}-{Bayes}: {A} {Frequentist} {Approach} to {Sequential} {Prior} {Updates} with {No} {Information} {Loss}.
\newblock 2025.

\bibitem[Xu and Raginsky(2017)]{xu_information-theoretic_2017}
Aolin Xu and Maxim Raginsky.
\newblock Information-theoretic analysis of generalization capability of learning algorithms.
\newblock 2017.

\bibitem[Zaslavsky et~al.(2021)Zaslavsky, Hu, and Levy]{zaslavsky-etal-2021-rate}
Noga Zaslavsky, Jennifer Hu, and Roger~P. Levy.
\newblock A {R}ate{--}{D}istortion view of human pragmatic reasoning?
\newblock In Allyson Ettinger, Ellie Pavlick, and Brandon Prickett, editors, \emph{Proceedings of the Society for Computation in Linguistics 2021}, pages 347--348, Online, February 2021. Association for Computational Linguistics.
\newblock URL \url{https://aclanthology.org/2021.scil-1.32/}.

\bibitem[Z{\'e}non et~al.(2019)Z{\'e}non, Solopchuk, and Pezzulo]{zenon_information-theoretic_2019}
Alexandre Z{\'e}non, Oleg Solopchuk, and Giovanni Pezzulo.
\newblock An information-theoretic perspective on the costs of cognition.
\newblock \emph{Neuropsychologia}, 123:\penalty0 5--18, February 2019.
\newblock ISSN 00283932.
\newblock \doi{10.1016/j.neuropsychologia.2018.09.013}.
\newblock URL \url{https://linkinghub.elsevier.com/retrieve/pii/S0028393218306328}.

\bibitem[Zheng et~al.(2023)Zheng, Chiang, Sheng, Zhuang, Wu, Zhuang, Lin, Li, Li, Xing, Zhang, Gonzalez, and Stoica]{NEURIPS2023_91f18a12}
Lianmin Zheng, Wei-Lin Chiang, Ying Sheng, Siyuan Zhuang, Zhanghao Wu, Yonghao Zhuang, Zi~Lin, Zhuohan Li, Dacheng Li, Eric Xing, Hao Zhang, Joseph~E Gonzalez, and Ion Stoica.
\newblock Judging llm-as-a-judge with mt-bench and chatbot arena.
\newblock In A.~Oh, T.~Naumann, A.~Globerson, K.~Saenko, M.~Hardt, and S.~Levine, editors, \emph{Advances in Neural Information Processing Systems}, volume~36, pages 46595--46623. Curran Associates, Inc., 2023.
\newblock URL \url{https://proceedings.neurips.cc/paper_files/paper/2023/file/91f18a1287b398d378ef22505bf41832-Paper-Datasets_and_Benchmarks.pdf}.

\bibitem[Ziegler et~al.(2020)Ziegler, Stiennon, Wu, Brown, Radford, Amodei, Christiano, and Irving]{ziegler_fine-tuning_2020}
Daniel~M. Ziegler, Nisan Stiennon, Jeffrey Wu, Tom~B. Brown, Alec Radford, Dario Amodei, Paul Christiano, and Geoffrey Irving.
\newblock Fine-{Tuning} {Language} {Models} from {Human} {Preferences}, January 2020.
\newblock URL \url{http://arxiv.org/abs/1909.08593}.
\newblock arXiv:1909.08593 [cs].

\end{thebibliography}

\newpage
\appendix

\section{Notation and Cross-References}
\label{app:notation}
Throughout, $\log$ denotes the natural logarithm.
Key references: single capacity inequality \eqref{eq:cascade_avg}; Fano–packing converse Thm.~\ref{thm:unified_fano}; PAC--Bayes Thm.~\ref{thm:pacbayes_basic}; expected KL identity Lemma~\ref{lem:kl_decomp}; dataset information Lemma~\ref{lem:dataset_info_chain}; capacity-control Proposition~\ref{prop:IUmtheta_capacity}; Posterior Bayes–Loss Identity Lemma~\ref{lem:bayes_transform}; Assumptions~\ref{assump:loss_index_link},~\ref{assump:residual}; final statements in Sec.~\ref{sec:interval}.

\section{Conditional Capacities and the Cascade}
\label{app:capacity}

\paragraph{Proof of Proposition~\ref{prop:cascade}.}
The cascade $U \to H \to Y$ given $S$ forms a Markov chain. The result follows directly from the data processing inequality \citep{shannon_mathematical_1948}, which states that for such a chain, $I(U;Y\mid S=s)\le I(U;H\mid S=s)$ and $I(U;Y\mid S=s)\le I(H;Y\mid S=s)$.
Taking suprema over the respective families yields $I(U;Y\mid S=s)\le C_{\mathrm{tot}\mid S}(s)$. Averaging over $S$ proves \eqref{eq:cascade_avg}.

\section{Packing Constructions for Common Losses}
\label{app:packing_constructions}

\subsection{Binary Classification}
\label{app:01}
Let $\mathcal U=[M]$ and $\mathcal A$ the set of labels. Take $a^{(i)}=i$. Then $\varepsilon=0$ and for any $j\neq i$, $\ell(u^{(j)},a^{(i)})=1$, giving $\Delta=1$. 
Let $\phi$ be the predicted label; Assump.~\ref{assump:loss_index_link} holds with margin $1$.

\subsection{Pairwise Ranking with 0–1 Loss}
\label{app:ranking}
Let $u^{(i)}$ encode a total order over items and $\ell$ be the fraction of misordered pairs. Use $a^{(i)}=u^{(i)}$ (predict that order). Then $\varepsilon=0$ and for any $j\neq i$, at least one pair flips, so $\Delta\ge 1/{\binom{n}{2}}$; with standard $\{0,1\}$ pairwise loss averaged over all $\binom{n}{2}$ pairs and normalized to $[0,1]$, the minimal separation is $\Delta=1/{\binom{n}{2}}$. Let $\phi$ output the predicted order; Assump.~\ref{assump:loss_index_link} holds.

\subsection{Truncated and Normalized MSE}
\label{app:mse}
Let $\ell(u,a)=\min\{\|u-a\|^2/\tau^2,1\}$. Choose an $r$-separated packing $\{u^{(i)}\}_{i=1}^M$ in $\mathcal U$ (under $\|\cdot\|$), and set $a^{(i)}=u^{(i)}$. Then $\varepsilon=0$ and, for $j\neq i$, the prototype cross-loss satisfies $\ell(u^{(j)},a^{(i)})\ge r^2/\tau^2$. 
To make Definition~\ref{def:separable_codebook} and Assumption~\ref{assump:loss_index_link} hold with a single margin, take the common choice
\[
\Delta\ =\ \frac{r^2}{4\tau^2}\, .
\]
Indeed, for any $a$ misclassified by the nearest-prototype Voronoi rule, one has $\|a-u^{(i)}\|\ge r/2$, so $\ell(u^{(i)},a)\ge r^2/(4\tau^2)=\Delta$. 
Since $r^2/\tau^2\ge \Delta$, the prototype cross-loss condition in Definition~\ref{def:separable_codebook} also holds.

\section{Fano–Packing Converse Details}
\label{app:fano}
We expand the proof of Thm.~\ref{thm:unified_fano}. 
Let $J$ be uniform on $[M]$, $U=U^{(J)}$. With $\hat J=\phi(\pi(Y,S),S)$, Lemma~\ref{lem:risk_to_error} gives $R(\pi)\ge (\varepsilon+\Delta)\, \mathbb P\{\hat J\neq J\}$. 
The standard form of Fano's inequality \citep{shannon_mathematical_1948}, when conditioned on $S$, implies that
$H(J\mid Y,S)\le \mathbb P\{\hat J\neq J\}\log(M-1)+h_2(\mathbb P\{\hat J\neq J\})$, which gives the more convenient bound
\[
\mathbb P\{\hat J\neq J\}\ \ge\ 1-\frac{I(J;Y\mid S)+\log 2}{\log M}.
\]
Using $J\to U\to Y$ given $S$ (Lemma~\ref{lem:J_to_U_to_Y}), we get \eqref{eq:unified_lower}; then apply \eqref{eq:cascade_avg} for \eqref{eq:unified_lower_capacity}.

\section{PAC--Bayes Details and Residual Control}
\label{app:residual}

\subsection{Proofs of Lemma~\ref{lem:kl_decomp} and Lemma~\ref{lem:dataset_info_chain}}
\label{app:pac_proofs}
\paragraph{Lemma~\ref{lem:kl_decomp}.}
The identity is a foundational result in information-theoretic learning theory \citep{xu_information-theoretic_2017, russo_how_2019}. The proof is as follows:
with $Q$ independent of $\mathcal D$,
$\mathbb E_{\mathcal D}[\mathrm{KL}(P\|Q)]=\mathbb E_{\mathcal D,\theta\sim P}\big[\log \tfrac{P(\theta\mid \mathcal D)}{Q(\theta)}\big]
= I(\mathcal D;\theta)+\mathrm{KL}(p(\theta)\|Q)$.

\paragraph{Lemma~\ref{lem:dataset_info_chain}.}
Data processing gives $I(U^m;\theta)\le I(U^m;\mathcal D)$. Then $I(U^m;Y^m,S^m)=I(U^m;S^m)+I(U^m;Y^m\mid S^m)$, 
with $I(U^m;S^m)=\sum_i I(U_i;S_i)=m I(U;S)$ by i.i.d.
For the conditional term, under the i.i.d.\ source and the memoryless channel $p(y_i\mid u_i,s_i)$, we have
$p(u^m\mid s^m)=\prod_i p(u_i\mid s_i)$ and hence 
$p(y^m\mid s^m)=\prod_i \!\int p(y_i\mid u_i,s_i)\,p(u_i\mid s_i)\,du_i=\prod_i p(y_i\mid s_i)$.
Therefore $H(Y^m\mid S^m)=\sum_i H(Y_i\mid S_i)$ and 
$H(Y^m\mid U^m,S^m)=\sum_i H(Y_i\mid U_i,S_i)$, which gives 
$I(U^m;Y^m\mid S^m)=\sum_i I(U_i;Y_i\mid S_i)$.

\subsection{Controlling the Residual Term}
\label{app:rho}
We list standard mechanisms to enforce Assump.~\ref{assump:residual}. These methods all serve to regularize the information that the learned parameters $\theta$ contain about the specific training dataset $\mathcal D$.
\emph{Algorithmic noise}: inject Gaussian noise into updates or use high-temperature posteriors; \emph{early stopping}: bound the mutual information by limiting the number of optimization steps; \emph{posterior smoothing}: mix the learned posterior with the prior. 
The general goal of controlling information flow, often framed as a form of compression, is a central theme in understanding deep learning generalization, although its precise role and benefits are still actively debated \citep{kawaguchi_how_2023, saxe_information_2019, shwartz-ziv_information-theoretic_2024, he_information-theoretic_2025}.

\section{From Expectation to High Probability}
\label{app:capacity_to_highprob}

This section provides a simple method to convert our expectation-based capacity bound on the KL-divergence into a high-probability statement. This type of conversion from expectation to high-probability bounds is a common step in applying learning-theoretic results. More sophisticated techniques can yield tighter, anytime-valid bounds that hold uniformly over time \citep{rodriguez-galvez_more_2024}. A direct application of Markov's inequality suffices.

\begin{lemma}[Markov Lift for the KL Term]
\label{lem:markov_kl}
Let $X\triangleq \mathrm{KL}(P\|Q)\ge 0$ denote the (dataset-dependent) PAC--Bayes KL term. 
For any $\eta\in(0,1)$, with probability at least $1-\eta$ (over the draw of $\mathcal D$),
\[
X\ \le\ \frac{\mathbb E_{\mathcal D}[X]}{\eta}.
\]
\end{lemma}
\begin{proof}
Since $X\ge 0$ and $\mathbb E_{\mathcal D}[X]<\infty$ under the conditions of Theorem~\ref{thm:pacbayes_basic}, Markov's inequality gives
\[
\mathbb P\!\left\{\,X>\frac{\mathbb E_{\mathcal D}[X]}{\eta}\,\right\}\ \le\ \eta.
\]
Equivalently, with probability at least $1-\eta$ we have $X\le \mathbb E_{\mathcal D}[X]/\eta$, as claimed.
\end{proof}

\begin{corollary}[A Capacity-Aware High-Probability Upper Bound]
\label{cor:capacity_highprob}
Fix $\delta,\eta\in(0,1)$. With probability at least $1-\delta-\eta$ (over the draw of $\mathcal D$), the PAC--Bayes bound of Thm.~\ref{thm:pacbayes_basic} implies
\[
\mathbb E_{\theta\sim P}\big[R_{\mathrm{obs}}(\theta)\big]
\ \le\ 
\mathbb E_{\theta\sim P}\big[\widehat R^{\mathrm{obs}}_m(\theta)\big]
\ +\ \sqrt{\frac{\mathbb E_{\mathcal D}[\mathrm{KL}(P\|Q)]/\eta+\log(1/\delta)}{2m}}.
\]
Combining with Cor.~\ref{cor:capacity_upper} and applying a union bound yields, with the same probability,
\[
\mathbb E_{\theta\sim P}\big[R_{\mathrm{obs}}(\theta)\big]
\ \le\ 
\mathbb E_{\theta\sim P}\big[\widehat R^{\mathrm{obs}}_m(\theta)\big]
\ \ +\ \sqrt{\frac{m\,\bar C_{\mathrm{tot}\mid S}+m\,I(U;S)+\rho+\mathrm{KL}(p(\theta)\|Q)}{2m\,\eta}\ +\ \frac{\log(1/\delta)}{2m}}\ .
\]
\emph{All constants are explicit; the price of eliminating the dataset randomness in $\mathrm{KL}(P\|Q)$ is the slack parameter $\eta$.}
\end{corollary}

\section{Loss Truncation and Normalization}
\label{app:truncation}
For unbounded losses such as MSE, define $\ell(u,a)=\min\{\|u-a\|^2/\tau^2,1\}$ for a scale $\tau>0$ (report $\tau$ when plotting). 
All PAC--Bayes statements and Thm.~\ref{thm:unified_fano} require only $\ell\in[0,1]$; truncation ensures this and keeps statements coordinate-free.

\section{Loss--Observable Link and Risk Transfer}
\label{app:loss_link}

\begin{assumption}[Loss--Observable Link]
\label{assump:loss_obs_link}
There exist constants $\alpha\ge 0$ and $\beta\ge 0$ such that for all measurable actions $a\in\mathcal A$ and all $(y,s)$ in the support of $(Y,S)$,
\begin{equation}
\label{eq:loss_obs_pointwise}
\mathbb E\!\big[\ell(U,a)\,\big|\,Y=y,\,S=s\big]\ \le\ \alpha\,\tilde\ell(y,s,a)\ +\ \beta\ .
\end{equation}
\end{assumption}

\begin{lemma}[Risk Transfer]
\label{lem:risk_transfer}
Under Assumption~\ref{assump:loss_obs_link}, for any (possibly randomized) decoder $\pi_\theta$,
\begin{equation}
\label{eq:risk_transfer}
\mathbb E\big[\ell\big(U,\pi_\theta(Y,S)\big)\big]\ \le\ \alpha\,\mathbb E\big[\tilde\ell\big(Y,S,\pi_\theta(Y,S)\big)\big]\ +\ \beta\ .
\end{equation}
In particular, when $\mathcal D$ is drawn from the same codebook-induced mixture distribution used in Theorem~\ref{thm:unified_fano}, taking $\theta\sim P(\cdot\mid\mathcal D)$ and expectation over both $\theta$ and the data gives
\[
\mathbb E_{\theta\sim P}\big[R_{\mathrm{mix}}(\pi_\theta)\big]\ \le\
\alpha\,\mathbb E_{\theta\sim P}\big[R_{\mathrm{obs}}(\theta)\big]\ +\ \beta\ .
\]
\end{lemma}

\begin{proof}
By the tower property and \eqref{eq:loss_obs_pointwise},
\[
\mathbb E\big[\ell\big(U,\pi_\theta(Y,S)\big)\big]
= \mathbb E\Big[\,\mathbb E\!\big[\ell(U,\pi_\theta(Y,S))\mid Y,S\big]\,\Big]
\le \alpha\,\mathbb E\big[\tilde\ell(Y,S,\pi_\theta(Y,S))\big]+\beta.
\]
Averaging over $\theta\sim P$ yields the stated forms.
\end{proof}

\section{Posterior Bayes Loss Identity}
\label{app:bayes_transform}

\begin{lemma}[Posterior Bayes–Loss Identity]
\label{lem:bayes_transform}
Fix any bounded loss $\ell\in[0,1]$ and define
$\tilde\ell^\star(y,s,a)\triangleq \mathbb E[\ell(U,a)\mid Y=y,\ S=s]$.
Then for any (possibly randomized) decoder $\pi_\theta$ and any data distribution over $(U,S,Y)$,
\[
\mathbb E\big[\tilde\ell^\star(Y,S,\pi_\theta(Y,S))\big]
\ =\ 
\mathbb E\big[\ell(U,\pi_\theta(Y,S))\big].
\]
\end{lemma}

\begin{proof}
By the tower property of conditional expectation,
$\mathbb E[\tilde\ell^\star(Y,S,\pi_\theta(Y,S))]
=\mathbb E\{\mathbb E[\ell(U,\pi_\theta(Y,S))\mid Y,S]\}
=\mathbb E[\ell(U,\pi_\theta(Y,S))]$.
\end{proof}

\section{Residual Control via Posterior Compression}
\label{app:posterior_compression}
Let $\theta\sim P(\cdot\mid\mathcal D)$ be the (possibly randomized) learner parameter.
Let $W$ be an auxiliary random seed, independent of $(U^m,S^m,Y^m)$.
Consider a data-independent randomized quantizer $T$ that maps $\theta$ to
$\tilde\theta=T(\theta,W)$ taking at most $K$ distinct values.
Let $P_c$ and $Q_c$ be the pushforwards of $P$ and $Q$ through $T$.
Then:
\begin{lemma}[Residual control by compression]
\label{lem:residual_compress}
$I(\mathcal D;\tilde\theta\mid U^m)\ \le\ H(\tilde\theta)\ \le\ \log K.$
\end{lemma}
\begin{proof}
$I(\mathcal D;\tilde\theta\mid U^m)\le H(\tilde\theta)$, and $H(\tilde\theta)\le\log K$ since
$\tilde\theta$ takes at most $K$ values.
\end{proof}
\begin{lemma}[KL does not increase under post-processing]
\label{lem:kl_post}
$\mathrm{KL}(P_c\|Q_c)\ \le\ \mathrm{KL}(P\|Q)$.
\end{lemma}
\noindent Using $P_c,Q_c$ in Theorem~\ref{thm:pacbayes_basic} and Lemma~\ref{lem:kl_decomp} yields the
capacity-aware bound of Corollary~\ref{cor:capacity_upper} with $\rho\le \log K$. This approach is conceptually related to other works that leverage model compression or selection of small representative subsets to derive non-vacuous generalization bounds for overparameterized models \citep{leblanc_generalization_2025, lotfi_non-vacuous_2024}.

\section{Context Coarsening}
\label{app:coarsen}
Let $S'=T(S)$ for a measurable (data-release) channel $T$ such that $U\to S\to S'$ forms a Markov chain (that is, $S'$ is generated from $S$ without direct access to $U$).
Then by data processing $I(U;S')\le I(U;S)$.
All definitions and bounds in the paper hold verbatim with $S'$ in place of $S$,
with capacities recomputed as $\bar C_{\mathrm{tot}\mid S'}$ and dataset information term
$m\,I(U;S')$ replacing $m\,I(U;S)$.
Thus, for any preprocessor $T$, Corollary~\ref{cor:capacity_upper} becomes
\[
\mathbb E_{\mathcal D}\!\big[\mathrm{KL}(P\|Q)\big]
\ \le\ m\,\bar C_{\mathrm{tot}\mid S'}\ +\ m\,I(U;S')\ +\ \rho\ +\ \mathrm{KL}(p(\theta)\|Q).
\]
Choosing $T$ to enforce $I(U;S')\le \kappa$ makes the interference term $m\,\kappa$ explicit.

\section{Soft Loss–Index Link}
\label{app:soft_link}
Assume there exists a measurable $\phi:\mathcal A\times\mathcal S\to [M]$ and parameters
$\varepsilon,\Delta\ge 0$, $\zeta\in[0,1)$ such that for all $i$ and all $a$,
\[
\mathbb P\big\{\,\mathbb E[\ell(U^{(i)},a)\mid S]\ \ge\ \varepsilon+\Delta\ \big|\ \phi(a,S)\neq i\,\big\}\ \ge\ 1-\zeta.
\]
Then for any decoder $\pi$ with $\hat J=\phi(\pi(Y,S),S)$, writing $E\!\triangleq\!\{\hat J\neq J\}$ and
\(
G\!\triangleq\!\big\{\mathbb E[\ell(U^{(J)},\pi(Y,S))\mid S]\ge \varepsilon+\Delta\big\},
\)
we have
\[
\mathbb E\big[\ell(U,\pi(Y,S))\big]\ 
=\ \mathbb E\!\Big[\,\mathbb E\big[\ell(U^{(J)},\pi(Y,S))\mid S\big]\cdot \mathbf 1_{E}\Big]\ +\ \mathbb E\!\Big[\,\mathbb E\big[\ell(U^{(J)},\pi(Y,S))\mid S\big]\cdot \mathbf 1_{E^c}\Big]
\]
\[
\ge\ (\varepsilon+\Delta)\,\mathbb P(E\cap G),
\]
hence
\[
\mathbb E\big[\ell(U,\pi(Y,S))\big]\ \ge\ (\varepsilon+\Delta)\,\mathbb P(E)\ -\ (\varepsilon+\Delta)\,\mathbb P(E\cap G^c).
\]
By the assumption, $\mathbb P(G^c\mid E)\le \zeta$, so $\mathbb P(E\cap G^c)\le \zeta\,\mathbb P(E)\le \zeta$ and consequently
\[
\mathbb E\big[\ell(U,\pi(Y,S))\big]\ \ge\ (\varepsilon+\Delta)\,\mathbb P\{\hat J\neq J\}\ -\ \zeta.
\]
A slightly tighter but equivalent multiplicative form also holds:
\[
\mathbb E\big[\ell(U,\pi(Y,S))\big]\ \ge\ (\varepsilon+\Delta)\,(1-\zeta)\,\mathbb P\{\hat J\neq J\}.
\]
Consequently, Theorem~\ref{thm:unified_fano} holds with an additive $-\zeta$ (or multiplicative $(1-\\zeta)$) slack in the lower bound.

\end{document}